\documentclass[11pt]{article}

\usepackage[T1]{fontenc}
\usepackage[utf8]{inputenc}
\usepackage{lmodern}
\usepackage[margin=1in]{geometry}
\usepackage{amsmath,amsfonts,amssymb,amsthm}
\usepackage{mathtools}
\usepackage{bbold}
\usepackage{xcolor}
\usepackage{booktabs}
\usepackage{placeins}
\usepackage{algorithm}
\usepackage{algpseudocode}
\usepackage[round,authoryear]{natbib}
\usepackage[hidelinks]{hyperref}
\usepackage[nameinlink,capitalize,noabbrev]{cleveref}
\usepackage{microtype}

\newcommand{\E}[1]{\mathbb{E}\left[#1\right]}
\newcommand{\Ex}[2]{\mathbb{E}_{#1}\left[#2\right]}

\renewcommand{\Pr}[1]{\mathbb{P}\left(#1\right)}

\newcommand{\Ind}[1]{\mathbb{1}\left[#1\right]}

\newcommand{\R}{\mathbb{R}}

\newcommand{\X}{\mathcal{X}}
\newcommand{\Y}{\mathcal{Y}}

\newcommand{\CEO}{\mathsf{E}}
\newcommand{\CEOd}{\CEO^{(d)}}
\newcommand{\kp}{\kappa}

\newcommand{\ourmethod}{\textbf{FSNM}}

\newcommand{\tr}[1]{\textrm{tr}\left\{ #1 \right\}}

\newcommand{\norm}[1]{\lVert #1 \rVert}

\newcommand{\independent}{\perp\!\!\!\perp}

\newtheorem{theorem}{Theorem}[section]
\newtheorem{proposition}[theorem]{Proposition}
\newtheorem{lemma}[theorem]{Lemma}

\newtheorem{assumption}[theorem]{Assumption}
\theoremstyle{definition}
\newtheorem{definition}[theorem]{Definition}
\theoremstyle{remark}

\renewenvironment{proof}[1][Proof]
  {\par\noindent\textbf{#1. }\ignorespaces}
  {\hfill$\blacksquare$\par\medskip}

\title{Learning Conditional Expectation Operators via Functional Newton Updates}

\author{
  Thiago R.~Ramos\thanks{Federal University of S\~ao Carlos. Email: \href{mailto:thiagorr@ufscar.br}{thiagorr@ufscar.br}.}
  \and
  Alek Fr\"{o}hlich\thanks{CSML, Istituto Italiano di Tecnologia and University of Genoa. Email: \href{mailto:alek.frohlich@iit.it}{alek.frohlich@iit.it}.}
  \and
  Daniel Perazzo\thanks{CSML, Istituto Italiano di Tecnologia and University of Genoa. Email: \href{mailto:daniel.rodrigues@iit.it}{daniel.rodrigues@iit.it}.}
  \and
  Massimiliano Pontil\thanks{CSML, Istituto Italiano di Tecnologia and University College London. Email: \href{mailto:massimiliano.pontil@iit.it}{massimiliano.pontil@iit.it}.}
}
\date{}

\begin{document}

\maketitle

\begin{abstract}
We introduce the Functional Spectral-Newton Method (FSNM) for learning the
leading singular structure of a conditional expectation operator without
fixing a basis or reproducing kernel Hilbert space. FSNM fits a low-rank
representation of the centered joint-to-product density ratio kernel by
alternating functional Newton updates. Each update reduces to a preconditioned
regression, which we approximate with vector-valued regression trees in a
stagewise boosting procedure. At the population level, we establish descent
and an $O(1/T)$ best-iterate block-stationarity rate under a relative
weak-learner accuracy condition, and show that every nondegenerate local
minimum over the full centered $L^2$ spaces is a globally optimal rank-$d$
approximation. Synthetic experiments show that FSNM recovers a low-rank
density ratio and its leading spectral structure, and that the same learned
kernel can answer multiple conditional queries without refitting.
\end{abstract}


\section{Introduction}
\label{sec:introduction}

Let $(X,Y)$ be a pair of random variables with joint distribution
$P_{X,Y}$. We define their conditional expectation operator as
\[
    \CEO:L^2(P_Y)\to L^2(P_X),
    \qquad
    \CEO(g)(x)=\E{g(Y)\mid X=x}.
\]
Learning $\CEO$, rather than a separate regression function for each choice of $g$, provides a reusable representation of the conditional law of $Y$ given $X$. For instance, taking $g$ to be an indicator function yields conditional probabilities, while taking $g$ to be the identity recovers the usual regression function; other choices yield conditional moments or distribution functions. This viewpoint has been exploited for conditional probability estimation and uncertainty quantification, both in general settings and in problems with known symmetries \citep{kostic2024ncp,ordonez-apraez2026representation}.
More broadly, conditional expectation operators describe
the evolution of observables in stochastic dynamical systems \citep{mardt2018vampnets,Klus2018,kostic2022koopman,kostic2023sharp,turri2026selfsupervised}, govern inverse problems arising in causal effect estimation \citep{wang2022spectral,sun2025spectral,
meunier2026outcomeaware}, characterize independence and conditional independence \citep{zhang_kernel-based_2012,frohlich2026toward}, and serve as world models for reinforcement learning \citep{novelli2024operator,zhang2022linear}.

Existing approaches to learning such operators typically use kernels or neural networks.
This raises a natural question:
\begin{quote}
\emph{Can the leading spectral structure of a conditional expectation operator be learned with regression trees, through a functional optimization procedure analogous to gradient boosting?}    
\end{quote}
Besides providing an alternative inductive bias, such a method would make operator learning available to the tree-based regression machinery commonly used for tabular and heterogeneous data.
In this work, we answer this question by deriving an alternating functional Newton method whose conditional expectation steps can be approximated by regression trees.

Our goal is to learn the leading spectral structure of $\CEO$ without
prescribing a basis for either $L^2$ space.
Throughout the paper, we assume that $P_{X,Y}$ is absolutely continuous with respect to $P_X\otimes P_Y$ and that the density ratio
\[
    \kp(x,y)
    =
    \frac{dP_{X,Y}}{d(P_X\otimes P_Y)}(x,y),
\]
belongs to $L^2(P_X\otimes P_Y)$.
Then $\CEO$ is a Hilbert--Schmidt integral operator with kernel $\kp$, and admits a singular value decomposition $\CEO=\sum_{i\geq0}\sigma_i\,\phi_i^\star\otimes\psi_i^\star$, with the
constant-one functions forming the first singular pair and $\sigma_0=1$.
The corresponding kernel expansion is
\[
    \kp(x,y)=1+\sum_{i\geq1}\sigma_i\phi_i^\star(x)\psi_i^\star(y).
\]
The unknown spectral structure is therefore carried by the centered density ratio kernel
\[
    \kp_0(x,y)\coloneqq \kp(x,y)-1
    =\sum_{i\geq1}\sigma_i\phi_i^\star(x)\psi_i^\star(y).
\]
Truncating this expansion to its leading $d$ terms motivates a
rank-$d$ model. We learn centered feature maps $\phi:\X\to\R^d$ and
$\psi:\Y\to\R^d$ such that
\[
    \kp_0(x,y)
    \approx
    \kp_{\phi,\psi}(x,y)
    \coloneqq
    \phi(x)^\top\psi(y),
    \qquad
    \E{\phi(X)}=\E{\psi(Y)}=0.
\]
The corresponding approximation of the full density ratio is
$1+\kp_{\phi,\psi}$.

At first sight, fitting this model appears to require knowing $\kp_0$.
The key observation is that its squared approximation error can instead be
written, up to an irrelevant constant, using a joint expectation over
observed pairs $(X,Y)$ and the marginal second moments of the two factors.
It can therefore be estimated directly from data.

To optimize this objective over flexible function classes, the
\emph{Functional Spectral-Newton Method} (\ourmethod) alternates functional
Newton steps for $\phi$ and $\psi$. For either factor, the functional
gradient is expressed through a conditional mean regression, while the
corresponding block Hessian is a finite-dimensional second-moment matrix.
Together, these estimates determine an approximate functional Newton
direction. The resulting boosting expansion accumulates weak regressors across
iterations
\citep{friedman2001greedy,grubb2011generalized,sigrist2021gradient,
zozoulenko2026gradient}.

Our main contributions are:
\begin{enumerate}
    \item Starting from a structured least-squares objective for low-rank approximation of the centered density ratio kernel, we derive its closed-form functional block-Newton directions. Each direction is a preconditioned conditional mean regression, which can be estimated by a generic regression learner without fixing a basis or an RKHS.

    \item We instantiate these updates with regression trees, obtaining an
    alternating stagewise procedure that accumulates weak regressors across
    iterations. We exploit the invariance of the factorization to obtain a
    balanced spectral representation of the fitted low-rank model, with
    orthonormal factors and explicit scale coefficients.

    \item We analyze a population version of the optimization procedure. Under
    uniformly nondegenerate iterates and a uniform relative
    tree-regression accuracy
    condition, the population loss decreases and a best-iterate
    block-stationarity measure has an $O(1/T)$ rate. Separately, over the full
    centered $L^2$ spaces, every nondegenerate local minimum of the population
    objective is a globally optimal rank-$d$ approximation.
    For the final fitted class, Rademacher complexity gives uniform
    concentration of the empirical loss around its population counterpart.
\end{enumerate}

\subsection{Related work}
\label{sec:related_work}

\paragraph{Learning conditional expectation operators.}
Conditional mean embeddings represent the conditional distribution of $Y$
given $X=x$ by its mean element in a reproducing kernel Hilbert space (RKHS),
or, equivalently, through an operator between RKHSs. They thereby allow the
conditional expectation of any function in the output RKHS to be evaluated
from a single learned representation; see \citet{muandet2017kernel} for a
comprehensive review. Recent work has learned reusable representations of
conditional laws more directly. Neural Conditional Probability uses an
operator-theoretic neural representation to recover conditional
probabilities, quantiles, moments, and confidence regions from a single fit
\citep{kostic2024ncp}, while related representations incorporate known
symmetries into statistical inference \citep{ordonez-apraez2026representation}.
Spectral features of conditional expectation operators have also been
used for the econometric and causal inference problems discussed above
\citep{wang2022spectral,sun2025spectral,
meunier2026outcomeaware,frohlich2026toward}. These methods, together
with conditional mean embeddings, share the goal of learning a reusable
representation of the conditional expectation operator, but typically fix an
RKHS or a neural parameterization, or tailor the objective to a particular
inferential problem. We instead target the leading singular system of
$\CEO$ directly in the ambient centered $L^2$ spaces, without choosing a
basis or RKHS, using a two-sided low-rank factorization fitted with
regression trees.

Answering several conditional queries from a single fit is
consequently not unique to \ourmethod: conditional mean embeddings and
Neural Conditional Probability share it by construction. What
distinguishes \ourmethod\ is how the representation is obtained and
analyzed: a closed-form, learner-agnostic block-Newton derivation
instantiated with regression trees, together with a matching convergence
and population landscape analysis for the resulting tree-boosting
procedure.

\paragraph{Density ratio estimation.}
The kernel $\kp$ is itself a density ratio, and estimating such ratios
directly, rather than the two densities separately, is a well-studied
problem. Classical approaches minimize a convex importance-fitting
objective, such as an unconstrained least-squares criterion, to obtain a
pointwise estimate of the ratio \citep{kanamori2009least}; see
\citet{sugiyama2012density} for a comprehensive treatment. These methods
target the full ratio $\kp$ directly and do not expose a low-rank or
spectral structure. A spectral series alternative uses eigenfunctions
of data-adapted kernel operators \citep{izbicki2014high}. We instead
estimate $\kp$ through a rank-$d$
factorization of its centered part $\kp_0$, which recovers the leading
singular functions and values of $\CEO$ and lets the same fitted kernel
answer multiple conditional queries without refitting.

\paragraph{Canonical correlation analysis.}
Canonical correlation analysis (CCA) seeks pairs of transformations of two
views whose outputs are maximally correlated \citep{hotelling1936relations}.
Nonlinear extensions include
kernel CCA, which restricts the transformations to RKHSs
\citep{bach2002kernel,fukumizu2007consistency}, and Deep CCA, which jointly
learns them with neural networks \citep{andrew2013deep}. Over unrestricted
centered $L^2$ spaces, these optimal transformations are the left and right
singular functions of the conditional expectation operator; related notions
include principal functions and principal inertia components
\citep{calmon2017principal,painsky2020nonlinear}. The rank-one case is also
closely related to ACE, which alternates conditional mean regressions to find
maximally correlated transformations \citep{breiman1985estimating}. Our
method instead starts from a joint rank-$d$ least-squares factorization of
{ the centered density ratio kernel}. Its conditional mean regressions arise
as functional block-Newton directions, can be fitted with trees, and jointly
produce the singular functions and their scale coefficients.

\paragraph{Functional gradient and Newton boosting.}
Boosting admits a classical interpretation as stagewise optimization in
function space, in which weak learners approximate functional descent
directions \citep{friedman2001greedy,grubb2011generalized}. Newton boosting
incorporates second-order information into this construction
\citep{sigrist2021gradient,zozoulenko2026gradient}. This second-order
approach underlies widely used gradient boosting systems such as XGBoost
\citep{chen2016xgboost} and LightGBM \citep{ke2017lightgbm}, which fit trees
to a per-example Newton step of a scalar loss. These methods typically
optimize a predictive loss for a scalar or finite-dimensional response.
\ourmethod\ instead optimizes a bilinear operator factorization: the block
Hessian is the opposite factor's $d\times d$ second-moment matrix, and
the Newton best response is a preconditioned conditional expectation on
the opposite domain. A fitted weak regressor approximates this best response
and thereby induces an approximate operator-update direction.

The remainder of the paper is organized as follows.
Section~\ref{sec:background} introduces the operator, its spectral
decomposition, and the low-rank population objective.
Section~\ref{sec:method} develops \ourmethod\ and its
spectral normalization,
Section~\ref{sec:theoretical_guarantees} presents the convergence,
population-landscape, and finite-sample guarantees, and
Section~\ref{sec:experiments} reports our
empirical evaluation.

\section{Background}
\label{sec:background}

Let $(X,Y)$ be a pair of random variables with joint distribution $P_{X,Y}$
and marginals $P_X$ and $P_Y$. Throughout, we assume access to an i.i.d.\
sample of size $n$,
\[
    \mathcal{D}_n=\{(X_i,Y_i)\}_{i=1}^n,
    \qquad
    (X_i,Y_i)\stackrel{\mathrm{i.i.d.}}{\sim}P_{X,Y}.
\]
A central problem in statistics and machine learning is to characterize
$P_{X,Y}$ beyond its marginals. A natural object for this purpose is the
conditional expectation operator: knowing its action gives access not only
to the conditional mean, but also to conditional distributions and other
conditional quantities.

\begin{definition}[Conditional Expectation Operator]\label{def:ceo}
Define the conditional expectation operator $\CEO:L^2(P_Y) \to L^2(P_X)$ by
\[
    \CEO(g)(x) \coloneqq \E{g(Y)\,\vert X=x}.
\]
\end{definition}

Here $L^2(P_X)$ and $L^2(P_Y)$ are the usual Hilbert spaces of
square-integrable functions, with inner products
\[
    \langle u_1,u_2\rangle_{L^2(P_X)}\coloneqq\Ex{P_X}{u_1(X)u_2(X)},
    \qquad
    \langle v_1,v_2\rangle_{L^2(P_Y)}\coloneqq\Ex{P_Y}{v_1(Y)v_2(Y)},
\]
for $u_1,u_2\in L^2(P_X)$ and $v_1,v_2\in L^2(P_Y)$, and induced norms
$\|u\|_{L^2(P_X)}\coloneqq\langle u,u\rangle_{L^2(P_X)}^{1/2}$ and
$\|v\|_{L^2(P_Y)}\coloneqq\langle v,v\rangle_{L^2(P_Y)}^{1/2}$.

Rather than learning a separate regression function for each choice of $g$,
we aim to learn the spectral decomposition of $\CEO$. Throughout, we assume
that $\CEO$ is compact and has finite Hilbert--Schmidt norm. For
$u\in L^2(P_X)$ and $v\in L^2(P_Y)$, let $u\otimes v$ denote the rank-one
operator
\[
    (u\otimes v)g
    \coloneqq \langle g,v\rangle_{L^2(P_Y)}u.
\]
Under this assumption, $\CEO$ admits a singular value decomposition.
Writing $\CEO^*:L^2(P_X)\to L^2(P_Y)$ for its adjoint,
\[
    \CEO^*(f)(y)\coloneqq\E{f(X)\mid Y=y},
\]
there exist a non-increasing sequence of nonnegative numbers
$\sigma_0\geq\sigma_1\geq\sigma_2\geq\cdots\geq0$ and orthonormal systems
$\{\phi_i^\star\}\subset L^2(P_X)$, $\{\psi_i^\star\}\subset L^2(P_Y)$ such
that, for every $i$,
\[
    \CEO\psi_i^\star=\sigma_i\phi_i^\star,
    \qquad
    \CEO^*\phi_i^\star=\sigma_i\psi_i^\star.
\]
We call $(\phi_i^\star,\psi_i^\star)$ a singular pair of $\CEO$ with
singular value $\sigma_i$; the systems $\{\phi_i^\star\}$ and
$\{\psi_i^\star\}$ are its left and right singular functions, respectively.
Equivalently,
\[
    \CEO
    =\sum_{i=0}^{\infty}
      \sigma_i\,\phi_i^\star\otimes\psi_i^\star.
\]

We now connect this operator decomposition to the joint distribution.
Assume that $P_{X,Y}\ll P_X\otimes P_Y$, and define
\[
    \kp(x,y)
    \coloneqq
    \frac{dP_{X,Y}}{d(P_X\otimes P_Y)}(x,y).
\]
Throughout, we assume that $\kp\in L^2(P_X\otimes P_Y)$. By a change of
measure, the action of $\CEO$ can be written as
\[
    \CEO(g)(x)
    =\int g(y)\kp(x,y)\,dP_Y(y).
\]
On the other hand, expanding the tensor products in the singular value
decomposition gives
\[
    \CEO(g)(x)
    =\int g(y)
      \left(
        \sum_{i=0}^{\infty}
        \sigma_i\phi_i^\star(x)\psi_i^\star(y)
      \right)dP_Y(y).
\]
Comparing the two expressions identifies the kernel of $\CEO$:
\[
    \kp(x,y)
    =\sum_{i=0}^{\infty}
      \sigma_i\phi_i^\star(x)\psi_i^\star(y),
\]
with convergence in $L^2(P_X\otimes P_Y)$.

The first singular component is known. Let $\mathbf{1}_X\in
L^2(P_X)$ and $\mathbf{1}_Y\in L^2(P_Y)$ denote the constant-one functions;
both have unit norm since $P_X$ and $P_Y$ are probability measures.
Conditional expectation is a contraction: by Jensen's inequality applied
conditionally on $X$,
\[
    \CEO(g)(x)^2
    =\E{g(Y)\mid X=x}^2
    \leq\E{g(Y)^2\mid X=x},
\]
and taking expectation over $X\sim P_X$ with the tower property gives
\[
    \|\CEO(g)\|_{L^2(P_X)}^2\leq\|g\|_{L^2(P_Y)}^2
    \quad\text{for every }g\in L^2(P_Y),
\]
so every singular value of $\CEO$ satisfies $\sigma_i\leq1$. Moreover,
\begin{align*}
    \CEO(\mathbf{1}_Y)(x)=\mathbf{1}_X(x),\qquad \CEO^*(\mathbf{1}_X)(y)=\mathbf{1}_Y(y),
\end{align*}
so $(\mathbf{1}_X,\mathbf{1}_Y)$ is a singular pair of $\CEO$ with value
$1$. Since no singular value exceeds $1$, this is the largest one, and we
may take $\sigma_0=1$, $\phi_0^\star=\mathbf{1}_X$, and
$\psi_0^\star=\mathbf{1}_Y$.
{ By orthogonality of the singular systems, every remaining singular
function is orthogonal to the corresponding constant-one function. Hence,
for $i\geq1$,
\[
    \Ex{P_X}{\phi_i^\star(X)}
    =\langle\phi_i^\star,\mathbf{1}_X\rangle_{L^2(P_X)}=0,
    \qquad
    \Ex{P_Y}{\psi_i^\star(Y)}
    =\langle\psi_i^\star,\mathbf{1}_Y\rangle_{L^2(P_Y)}=0.
\]
Thus, the nonconstant singular functions are centered, and} the two spectral
representations become
\[
    \CEO
    =\mathbf{1}_X\otimes\mathbf{1}_Y
      +\sum_{i=1}^{\infty}
       \sigma_i\,\phi_i^\star\otimes\psi_i^\star,
    \quad
    \kp(x,y)
    =1+\sum_{i=1}^{\infty}
      \sigma_i\phi_i^\star(x)\psi_i^\star(y).
\]
{ Define the centered spaces
\[
    L_0^2(P_X)\coloneqq
    \{f\in L^2(P_X):\Ex{P_X}{f(X)}=0\},
    \qquad
    L_0^2(P_Y)\coloneqq
    \{g\in L^2(P_Y):\Ex{P_Y}{g(Y)}=0\},
\]
and the centered density ratio kernel
\[
    \kp_0(x,y)\coloneqq\kp(x,y)-1.
\]
The marginal identities of a density ratio imply
\[
    \Ex{P_Y}{\kp_0(x,Y)}=0
    \quad\text{for $P_X$-a.e. }x,
    \qquad
    \Ex{P_X}{\kp_0(X,y)}=0
    \quad\text{for $P_Y$-a.e. }y.
\]
Hence $\kp_0$ lies in the centered tensor-product subspace and has the
spectral representation
\[
    \kp_0(x,y)
    =\sum_{i=1}^{\infty}
      \sigma_i\phi_i^\star(x)\psi_i^\star(y).
\]
Equivalently, the unknown part of the conditional expectation operator is
the centered operator
\[
    \CEO_0\coloneqq
    \CEO-\mathbf{1}_X\otimes\mathbf{1}_Y:
    L^2(P_Y)\to L^2(P_X),
\]
whose kernel is $\kp_0$. It annihilates constants and maps
$L_0^2(P_Y)$ into $L_0^2(P_X)$. Thus, learning the nonconstant singular
structure of $\CEO$ is equivalent to learning a separable decomposition of
$\kp_0$.}

The spectral representation also characterizes the optimal finite-rank
approximation.

{
\begin{theorem}[Eckart--Young--Mirsky]\label{thm:optimal_truncation}
For $d\geq0$, define
\[
    \CEO_0^{(d)}
    \coloneqq
      \sum_{i=1}^{d}
       \sigma_i\,\phi_i^\star\otimes\psi_i^\star,
    \qquad
    \kp_0^{(d)}(x,y)
    \coloneqq \sum_{i=1}^{d}
       \sigma_i\phi_i^\star(x)\psi_i^\star(y).
\]
Then $\CEO_0^{(d)}$ and $\kp_0^{(d)}$ are best rank-$d$
approximations of $\CEO_0$ and $\kp_0$, respectively, and
\[
    \|\CEO_0-\CEO_0^{(d)}\|_{\mathrm{HS}}^2
    =
    \|\kp_0-\kp_0^{(d)}\|_{L^2(P_X\otimes P_Y)}^2
    =\sum_{i>d}\sigma_i^2.
\]
Adding the known constant component gives
\[
    \CEOd
    =\mathbf{1}_X\otimes\mathbf{1}_Y+\CEO_0^{(d)},
    \qquad
    \kp^{(d)}=1+\kp_0^{(d)},
\]
which are the corresponding approximations of the full operator and its
density ratio kernel.
\end{theorem}

Thus, learning the first $d$ unknown singular components is exactly the
problem of finding the best rank-$d$ approximation of $\CEO_0$, or
equivalently of $\kp_0$.
}


{
Theorem~\ref{thm:optimal_truncation} identifies the spectral target.
We now parameterize this rank-$d$ approximation and derive a population
objective that can be expressed through observable expectations. In
particular, the optimal rank-$d$ truncation of the centered density ratio
kernel is
\[
    \kp_0^{(d)}(x,y)
    =\sum_{i=1}^d
      \sigma_i\phi_i^\star(x)\psi_i^\star(y).
\]
Rather than estimating the singular values and normalized singular functions
separately, we absorb each coefficient into the corresponding pair of
factors. In particular, setting
\[
    \phi_i=\sqrt{\sigma_i}\phi_i^\star,
    \qquad
    \psi_i=\sqrt{\sigma_i}\psi_i^\star
\]
rewrites the spectral truncation without explicit coefficients. This
motivates the rank-$d$ model
\[
    \kp_{\phi,\psi}(x,y)
    \coloneqq
    \sum_{i=1}^d\phi_i(x)\psi_i(y)
    =\phi(x)^\top\psi(y),
\]
where
$
    \phi\in L_0^2(P_X)^d,
    \,
    \psi\in L_0^2(P_Y)^d.
$
Thus each coordinate of both factors is centered.

For centered function classes
$\Phi\subseteq L_0^2(P_X)^d$ and
$\Psi\subseteq L_0^2(P_Y)^d$, the learning problem is
\[
    \min_{\phi\in\Phi,\,\psi\in\Psi}
    \|\kp_0-\kp_{\phi,\psi}\|_{L^2(P_X\otimes P_Y)}^2.
\]

Although $\kp_0$ is unknown, its relation to the density ratio allows the
approximation error to be expressed using observable expectations.

\begin{proposition}\label{prop:population_loss}
Let
$
    \Sigma_\phi\coloneqq\Ex{P_X}{\phi(X)\phi(X)^\top},
    \Sigma_\psi\coloneqq\Ex{P_Y}{\psi(Y)\psi(Y)^\top},
$
and define
\begin{align*}
    A(\phi,\psi)
    &\coloneqq \Ex{P_{X,Y}}{\phi(X)^\top\psi(Y)},\\
    C(\phi,\psi)
    &\coloneqq
        \Ex{P_X\otimes P_Y}{\bigl(\phi(X)^\top\psi(Y)\bigr)^2}=\tr{\Sigma_\phi\Sigma_\psi},
\end{align*}

Then, the loss function can be written as
\[
    \|\kp_0-\kp_{\phi,\psi}\|_{L^2(P_X\otimes P_Y)}^2
    =
    \|\kp_0\|_{L^2(P_X\otimes P_Y)}^2
    +C(\phi,\psi)-2A(\phi,\psi).
\]
Consequently, the approximation problem is equivalent to minimizing
\[
    L(\phi,\psi)
    \coloneqq C(\phi,\psi)-2A(\phi,\psi)
    =\tr{\Sigma_\phi\Sigma_\psi}
      -2\Ex{P_{X,Y}}{\phi(X)^\top\psi(Y)}.
\]
\end{proposition}

A proof is provided in Appendix~\ref{app:proofs_model_objective}.
}

\section{The Functional Spectral-Newton method}\label{sec:method}

{ Minimizing $L$ over $\phi\in L_0^2(P_X)^d$ and
$\psi\in L_0^2(P_Y)^d$ is an infinite-dimensional optimization problem.}
Nevertheless, its low-rank factorization gives it a useful block structure:
with one factor fixed, the loss is quadratic in the other. This motivates
our central idea, that is, a functional Newton method that alternates between the two
factors, taking a Newton step in function space for each block.

To make this structure explicit, we first compute the functional gradients
and block Hessians of $L$. These derivatives will show that each block update
is a conditional expectation regression, preconditioned by a second-moment
matrix, and can therefore be estimated from the sample $\mathcal{D}_n$.
We then derive the resulting closed-form block updates and introduce a
spectral normalization to resolve the non-uniqueness of the factorization.
We conclude with the practical finite-sample implementation used in our
experiments.

\subsection{Functional Newton framework}\label{sec:functional_calculus}

%

{ We introduce the following quantities, which appear throughout the analysis:
\begin{gather}
   m_\psi(x) \coloneqq
      \E{\psi(Y)\mid X=x},
   \quad
   m_\phi(y) \coloneqq
      \E{\phi(X)\mid Y=y},
   \label{eq:conditional_moments}\\
   \Sigma_\psi \coloneqq \Ex{P_Y}{\psi(Y)\psi(Y)^\top},
   \quad
   \Sigma_\phi \coloneqq \Ex{P_X}{\phi(X)\phi(X)^\top}.
   \label{eq:second_moments}
\end{gather}
Because the factors are centered, the tower property gives
$\E{m_\psi(X)}=\E{m_\phi(Y)}=0$, so both conditional mean functions belong
to the corresponding centered spaces.}

We can now state the main result of this section.

{
\begin{proposition}
\label{prop:derivatives}
With $m_\psi, m_\phi, \Sigma_\psi, \Sigma_\phi$
as in~\eqref{eq:conditional_moments}--\eqref{eq:second_moments}, the
functional gradients of $L$ are
\[
    \nabla_\phi L(\phi,\psi)(x) = 2\,\Sigma_{\psi}\,\phi(x) - 2\,m_\psi(x),
    \quad
    \nabla_\psi L(\phi,\psi)(y) = 2\,\Sigma_{\phi}\,\psi(y) - 2\,m_\phi(y).
\]
The diagonal Hessian blocks are the constant multiplication operators
\[
    \nabla^2_{\phi\phi}L(\phi,\psi)=2\Sigma_\psi,
    \quad
    \nabla^2_{\psi\psi}L(\phi,\psi)=2\Sigma_\phi.
\]
\end{proposition}
}

The proof is provided in
Appendix~\ref{app:proofs_functional_calculus}.
By Proposition~\ref{prop:derivatives}, fixing either factor makes $L$ a
quadratic functional of the other, with a constant block Hessian. We
alternate between the two factors in Gauss--Seidel order: we first update
$\psi$ using $\phi_t$ and then update $\phi$ using the new factor
$\psi_{t+1}$. The relaxed functional Newton steps are
\begin{align*}
\psi_{t+1}
= \psi_t
- \frac{\eta_\psi}{2}\, \Sigma_{\phi,t}^{-1}
\nabla_\psi L(\phi_t,\psi_t),\qquad\phi_{t+1}
= \phi_t
- \frac{\eta_\phi}{2}\, \Sigma_{\psi,t+1}^{-1}
\nabla_\phi L(\phi_t,\psi_{t+1}),
\end{align*}
with relaxation parameters $\eta_\phi,\eta_\psi \in (0,1]$.

To ensure that both Newton steps are well defined and uniformly stable, we
impose the following assumption.

\begin{assumption}[Uniformly nondegenerate iterates]
\label{ass:nondegenerate_factors}
There exist constants $0<\lambda\le\Lambda<\infty$ such that, at every
iteration $t$, the second-moment matrices required by the two block updates
satisfy
\[
    \lambda I\preceq\Sigma_{\phi,t}\preceq\Lambda I,
    \quad
    \lambda I\preceq\Sigma_{\psi,t+1}\preceq\Lambda I.
\]
\end{assumption}

{Assumption~\ref{ass:nondegenerate_factors} is a regularity condition
that keeps the unregularized population analysis tractable.
Because each block subproblem is quadratic, its Newton step points directly
to its unique minimizer. Indeed, fixing $\phi=\phi_t$,
Proposition~\ref{prop:derivatives} gives
$\nabla_\psi L(\phi_t,\psi)(y)=2\Sigma_{\phi,t}\psi(y)-2m_{\phi,t}(y)$;
setting this gradient to zero and using $\Sigma_{\phi,t}\succ0$
(Assumption~\ref{ass:nondegenerate_factors}) yields the unique block
minimizer
\[
    \psi_{t+1}^*(y)
    \coloneqq \Sigma_{\phi,t}^{-1}m_{\phi,t}(y).
\]
Analogously, fixing $\psi=\psi_{t+1}$ and setting
$\nabla_\phi L(\phi,\psi_{t+1})=0$ gives the block minimizer
\[
    \phi_{t+1}^*(x)
    \coloneqq \Sigma_{\psi,t+1}^{-1}m_{\psi,t+1}(x).
\]
Substituting the gradients from Proposition~\ref{prop:derivatives} into the
Newton steps gives the equivalent and more revealing form
\[
    \psi_{t+1}
    =(1-\eta_\psi)\psi_t+\eta_\psi\psi_{t+1}^*,
    \quad
    \phi_{t+1}
    =(1-\eta_\phi)\phi_t+\eta_\phi\phi_{t+1}^*.
\]
Thus, each update is a convex combination of the current factor and its best
response. When $\eta_\phi=\eta_\psi=1$, each step exactly minimizes the
current block subproblem; smaller values give damped block-Newton steps.

\subsection{Spectral normalization}

The factors produced by the block updates approximate $\kp_0$ well,
but their coordinates need not be orthogonal, normalized, or aligned with
the singular directions: $\Sigma_\phi$ and $\Sigma_\psi$ need not even be
equal. Recovering an explicit spectral representation therefore requires a
change of basis.

Since the model and the loss depend on $(\phi,\psi)$ only through the
pointwise product $\phi(x)^\top\psi(y)$, the reparametrization
\[
    (\phi,\psi)\longmapsto(A\phi,A^{-\top}\psi)
\]
leaves { $\kp_{\phi,\psi}$} and $L$ unchanged for every invertible
$A\in\R^{d\times d}$, since
$(A\phi(x))^\top(A^{-\top}\psi(y))=\phi(x)^\top\psi(y)$. This freedom is
exactly what is needed: choosing $A$ to make the second-moment matrices
diagonal and equal, and then rescaling, turns $(\phi,\psi)$ into an
equivalent pair with orthonormal coordinates, without changing the fitted
kernel or the loss.

\begin{proposition}\label{prop:balancing}
Let $\Sigma_\phi,\Sigma_\psi\succ0$ and let
$\Sigma_\phi^{1/2}\Sigma_\psi\Sigma_\phi^{1/2}=U\Lambda U^\top$ be an eigendecomposition
with $U$ orthogonal and $\Lambda=\operatorname{diag}(s_1^2,\dots,s_d^2)$,
$s_1\ge\cdots\ge s_d>0$. Set $A=\Lambda^{1/4}U^\top\Sigma_\phi^{-1/2}$,
$\tilde\phi=A\phi$, and $\tilde\psi=A^{-\top}\psi$. Then
\[
    \Sigma_{\tilde\phi}=\Sigma_{\tilde\psi}=\operatorname{diag}(s_1,\dots,s_d),
    \qquad \tilde\phi^\top\tilde\psi=\phi^\top\psi,
\]
so { $\kp_{\tilde\phi,\tilde\psi}=\kp_{\phi,\psi}$} and
$L(\tilde\phi,\tilde\psi)=L(\phi,\psi)$, while
$s_i=\sqrt{\lambda_i(\Sigma_\phi\Sigma_\psi)}$
are the singular values of { $\kp_{\phi,\psi}$}. Writing $\tilde\phi=D^{1/2}\phi^\perp$
and $\tilde\psi=D^{1/2}\psi^\perp$ with $D=\operatorname{diag}(s_i)$ gives orthonormal
$\phi^\perp,\psi^\perp$ and recovers the SVD form
{ $\kp_{\phi,\psi}=\sum_{i=1}^d s_i\phi_i^\perp(x)\psi_i^\perp(y)$}.
\end{proposition}

The proof is provided in Appendix~\ref{app:proofs_method}.
{ When $\kp_{\phi,\psi}$ is an optimal rank-$d$ truncation of
$\kp_0$, the learned values $s_i$ and the normalized factors coincide
with a leading singular system of $\CEO_0$.}

The block updates commute with the reparametrization
$(\phi,\psi)\mapsto(A\phi,A^{-\top}\psi)$. Indeed, since $m_\phi$ is linear
and $\Sigma_\phi$ is bilinear in $\phi$,
\[
    m_{A\phi}=Am_\phi,
    \qquad
    \Sigma_{A\phi}=A\Sigma_\phi A^\top,
\]
so that, for every invertible $A$,
\[
    \Sigma_{A\phi}^{-1}m_{A\phi}
    =(A\Sigma_\phi A^\top)^{-1}(Am_\phi)
    =A^{-\top}\Sigma_\phi^{-1}m_\phi,
\]
and, analogously,
$\Sigma_{A^{-\top}\psi}^{-1}m_{A^{-\top}\psi}=A\Sigma_\psi^{-1}m_\psi$.
Consequently, the block update computed from a reparametrized factor equals
the correspondingly reparametrized update computed from the original
factor.
Hence reparametrizing the current iterate before or after a block update
produces the same pair up to the change of basis, so the represented kernel
$\kp_t=\phi_t^\top\psi_t$ and the loss $L_t$ are unaffected by when the
balancing transform of Proposition~\ref{prop:balancing} is applied. Algorithm~\ref{alg:fsnm} applies this transform once, after the
final iteration: diagonalizing and equalizing $\Sigma_{\phi_T}$ and
$\Sigma_{\psi_T}$ while extracting the singular values $s_i$ and the
orthonormal factors $\phi^\perp,\psi^\perp$, to construct the returned
spectral representation. Section~\ref{sec:practical_implementation} describes a practical
variant that instead balances after every iteration, motivated by
finite-sample considerations that fall outside the equivariance argument
above.


\begin{algorithm}[htbp]
\caption{Functional Spectral-Newton Method (FSNM)}
\label{alg:fsnm}
\begin{algorithmic}[1]
\Require { centered initial $\phi_0,\psi_0$}; step sizes $\eta_\phi,\eta_\psi\in(0,1]$; iterations $T$
\For{$t=0,1,\ldots,T-1$}
    \State $\psi_{t+1}(y) = (1-\eta_\psi)\,\psi_t(y) + \eta_\psi\,\Sigma_{\phi,t}^{-1}m_{\phi,t}(y)$
    \State $\phi_{t+1}(x) = (1-\eta_\phi)\,\phi_t(x) + \eta_\phi\,\Sigma_{\psi,t+1}^{-1}m_{\psi,t+1}(x)$
\EndFor
\State let $A_T$ be the balancing matrix of
    Proposition~\ref{prop:balancing} for $(\Sigma_{\phi,T},\Sigma_{\psi,T})$
\State $\phi_T \gets A_T\,\phi_T$, $\qquad \psi_T \gets A_T^{-\top}\,\psi_T$
\State $D \gets \Sigma_{\phi,T}$ \Comment{$(\phi_T,\psi_T)$ is now balanced by the preceding step: $\Sigma_{\phi,T}=\Sigma_{\psi,T}=D$}
\State $\phi^\perp \gets D^{-1/2}\phi_T, \qquad \psi^\perp \gets D^{-1/2}\psi_T$
\Ensure spectral representation $(s_i,\phi_i^\perp,\psi_i^\perp)_{i=1}^d$ with $s_i=D_{ii}$
\end{algorithmic}
\end{algorithm}
\FloatBarrier

\subsection{Practical implementation}
\label{sec:practical_implementation}

Each iteration of Algorithm~\ref{alg:fsnm} has two finite-sample steps: the
$\psi$-block update and the $\phi$-block update. We estimate the
conditional means and second-moment matrices of
\eqref{eq:conditional_moments}--\eqref{eq:second_moments} from the sample
$\mathcal D_n$, and describe each step in turn. Our implementation
adds a third, per-iteration balancing step beyond Algorithm~\ref{alg:fsnm}.

\paragraph{$\psi$-block update.}
Given $\phi_t$, we estimate its second-moment matrix by
\[
    \widehat\Sigma_{\phi,t}
    =\frac1n\sum_{i=1}^n\phi_t(X_i)\phi_t(X_i)^\top,
\]
and, for each coordinate $j=1,\ldots,d$, fit a regression tree on $Y$ to
the centered values of $\phi_t$:
\[
    \widehat m_{\phi,t,j}
    \in\arg\min_{g\,\text{tree on }\Y}
    \frac1n\sum_{i=1}^n
    \bigl(\phi_{t,j}(X_i)-\bar\phi_{t,j}-g(Y_i)\bigr)^2,
    \qquad
    \bar\phi_{t}=\frac1n\sum_{i=1}^n\phi_t(X_i).
\]
Stacking these $d$ fits into $\widehat m_{\phi,t}$ gives the finite-sample
update
\[
    \psi_{t+1}
    =(1-\eta_\psi)\psi_t
    +\eta_\psi(\widehat\Sigma_{\phi,t}+\rho I_d)^{-1}
      \widehat m_{\phi,t}.
\]
The ridge term $\rho I_d$ is added before inverting $\widehat\Sigma_{\phi,t}$:
with a finite sample the estimated second-moment matrix can be
ill-conditioned or nearly singular---for instance early in training, or
whenever a fitted coordinate of $\phi_t$ is nearly constant---and the ridge
term keeps this inversion numerically stable.

\paragraph{$\phi$-block update.}
The $\phi$-block is estimated symmetrically from the new factor
$\psi_{t+1}$: with
\[
    \bar\psi_{t+1}=\frac1n\sum_{i=1}^n\psi_{t+1}(Y_i),
    \qquad
    \widehat\Sigma_{\psi,t+1}
    =\frac1n\sum_{i=1}^n\psi_{t+1}(Y_i)\psi_{t+1}(Y_i)^\top,
\]
we fit, for $j=1,\ldots,d$,
\[
    \widehat m_{\psi,t+1,j}
    \in\arg\min_{h\,\text{tree on }\X}
    \frac1n\sum_{i=1}^n
    \bigl(\psi_{t+1,j}(Y_i)-\bar\psi_{t+1,j}-h(X_i)\bigr)^2,
\]
and update
\[
    \phi_{t+1}
    =(1-\eta_\phi)\phi_t
    +\eta_\phi(\widehat\Sigma_{\psi,t+1}+\rho I_d)^{-1}
      \widehat m_{\psi,t+1}.
\]

\paragraph{Balancing.}
Algorithm~\ref{alg:fsnm} balances only after the final iteration,
whereas our implementation balances after every iteration using empirical
second moments. This heuristic is outside the theory: although exact
balancing preserves the kernel and loss by Proposition~\ref{prop:balancing},
ridge regularization and coordinatewise tree fitting are not exactly
equivariant under general reparametrizations. Per-iteration balancing
improves conditioning; in the rank-three experiment of
Section~\ref{sec:rank_three_experiment}, balancing only at the end increases
kernel RMSE from $0.1204$ to $0.1944$, subspace error from $0.3340$ to
$0.5897$, and the largest condition number from $6.14$ to $22.06$. We
therefore use per-iteration balancing with $\rho=10^{-8}$; its final
application returns the spectral representation, while
Section~\ref{sec:finite_sample_guarantee} gives a qualitative generalization
guarantee whenever the returned kernel admits a factorization in the stated
convex-hull classes.

\section{Theoretical Guarantees}
\label{sec:theoretical_guarantees}

This section establishes three complementary properties of \ourmethod.
First, the population weak-learner updates decrease the objective and approach
block stationarity at an $O(1/T)$ best-iterate rate. Second, the stationary
points and local minima of the unrestricted population objective admit a
spectral characterization. Third, a uniform-concentration argument gives a
finite-sample generalization bound for the final fitted model.

We use distinct notation for population tree approximations and
sample-fitted regressions. Let $\mathcal T_X^0\subset L_0^2(P_X)^d$ and
$\mathcal T_Y^0\subset L_0^2(P_Y)^d$ denote the classes of
population-centered, vector-valued regression trees on $\mathcal X$ and
$\mathcal Y$, respectively. Each element stacks $d$ scalar trees, one per
coordinate, from a prescribed split family. At iteration $t$, define the
population tree regressions by
\begin{align*}
    m_{\phi,t}^{\mathcal T}
    &\in
    \arg\min_{g\in\mathcal T_Y^0}
    \E{\|\phi_t(X)-g(Y)\|_2^2},
    &
    m_{\psi,t+1}^{\mathcal T}
    &\in
    \arg\min_{h\in\mathcal T_X^0}
    \E{\|\psi_{t+1}(Y)-h(X)\|_2^2}.
\end{align*}

The unrestricted conditional means $m_{\phi,t}$ and
$m_{\psi,t+1}$, and their exact block best responses
$\psi_{t+1}^*$ and $\phi_{t+1}^*$, were defined in
Section~\ref{sec:method}. Their tree-restricted counterparts are
\[
    \psi_{t+1}^{\mathcal T,*}
    \coloneqq\Sigma_{\phi,t}^{-1}m_{\phi,t}^{\mathcal T},
    \qquad
    \phi_{t+1}^{\mathcal T,*}
    \coloneqq\Sigma_{\psi,t+1}^{-1}m_{\psi,t+1}^{\mathcal T}.
\]

All these functions are centered. The exact responses serve only as
blockwise benchmarks along the tree-generated trajectory; they do not define
a separate sequence of iterates. As in Algorithm~\ref{alg:fsnm}, balancing is
applied after the final iteration and does not enter the optimization
dynamics analyzed below.

\subsection{Convergence to block stationarity}

We take the initial factors to be centered. Centering is preserved by
the population updates, while no balancing condition is imposed on the
intermediate iterates. Assumption~\ref{ass:nondegenerate_factors} controls
their second-moment matrices directly.
Define their associated block-optimality gaps before the two updates by
\[
    G_t^\psi
    \coloneqq L(\phi_t,\psi_t)-L(\phi_t,\psi_{t+1}^*),
    \quad
    G_t^\phi
    \coloneqq
    { L(\phi_t,\psi_{t+1})
    -L(\phi_{t+1}^*,\psi_{t+1})}.
\]
Both gaps are nonnegative because they compare the current factor with its
exact block minimizer.

The difference between the unrestricted and tree-restricted regressions
determines the error in the approximate Newton direction. For a
positive-definite matrix $M$ and a vector-valued function $f$, write
\[
    \|f\|_M^2\coloneqq \E{f(Z)^\top M f(Z)},
\]
where the distribution of $Z$ is determined by the domain of $f$.
To compare the oracle and tree-based updates directly, let
\[
    r_t^\psi\coloneqq\psi_{t+1}^*-\psi_t,
    \quad
    r_t^\phi\coloneqq\phi_{t+1}^*-\phi_t
\]
be the displacements from the current factors to their exact block
minimizers, or equivalently the exact Newton directions. Replacing the oracle
responses by the fitted weak-learner responses gives the directions actually
used by the approximate updates:
\[
    h_t^\psi\coloneqq\psi_{t+1}^{\mathcal T,*}-\psi_t,
    \quad
    h_t^\phi\coloneqq\phi_{t+1}^{\mathcal T,*}-\phi_t.
\]
By the definitions above, their approximation errors satisfy

\begin{align*}
    r_t^\psi-h_t^\psi
    =\Sigma_{\phi,t}^{-1}
      (m_{\phi,t}-m_{\phi,t}^{\mathcal T}),\qquad r_t^\phi-h_t^\phi
    =\Sigma_{\psi,t+1}^{-1}
      (m_{\psi,t+1}-m_{\psi,t+1}^{\mathcal T}).
\end{align*}

With the gaps, directions, and approximation errors now in place, the next
result characterizes the gaps in closed form and records the consequences
used throughout the convergence analysis.

\begin{lemma}[Block-optimality gap]
\label{lem:gap_characterization}
For every $\psi\in L_0^2(P_Y)^d$ and $\phi\in L_0^2(P_X)^d$,
\[
    L(\phi_t,\psi)-L(\phi_t,\psi_{t+1}^*)
    =\|\psi-\psi_{t+1}^*\|_{\Sigma_{\phi,t}}^2,
    \qquad
    L(\phi,\psi_{t+1})-L(\phi_{t+1}^*,\psi_{t+1})
    =\|\phi-\phi_{t+1}^*\|_{\Sigma_{\psi,t+1}}^2.
\]
In particular, $G_t^\psi=\|r_t^\psi\|_{\Sigma_{\phi,t}}^2$ and
$G_t^\phi=\|r_t^\phi\|_{\Sigma_{\psi,t+1}}^2$, and, under
Assumption~\ref{ass:nondegenerate_factors},
\[
    \lambda\|\psi_t-\psi_{t+1}^*\|_{L^2(P_Y)}^2
    \le G_t^\psi\le
    \Lambda\|\psi_t-\psi_{t+1}^*\|_{L^2(P_Y)}^2,
\]
with the analogous sandwich for $G_t^\phi$ and
$\phi_t-\phi_{t+1}^*$. Moreover,
\[
    \nabla_\psi L(\phi_t,\psi_t)=2\Sigma_{\phi,t}(\psi_t-\psi_{t+1}^*),
    \qquad
    \|\nabla_\psi L(\phi_t,\psi_t)\|_{L^2(P_Y)}^2\le4\Lambda G_t^\psi,
\]
with the analogous statements for $\nabla_\phi L(\phi_t,\psi_{t+1})$ and
$G_t^\phi$.
\end{lemma}

{
\begin{assumption}[Relative tree-regression accuracy]
\label{ass:relative_tree_accuracy}
There exist $\gamma_\phi,\gamma_\psi\in(0,1]$ such that, at every iterate
generated by the population updates,
\begin{align*}
    \inf_{g\in\mathcal T_Y^0}
    \|m_{\phi,t}-g\|_{L^2(P_Y)}^2
    \le
    \lambda(1-\gamma_\psi^2)G_t^\psi,\quad    \inf_{h\in\mathcal T_X^0}
    \|m_{\psi,t+1}-h\|_{L^2(P_X)}^2
    \le
    \lambda(1-\gamma_\phi^2)G_t^\phi.
\end{align*}
\end{assumption}

Assumption~\ref{ass:relative_tree_accuracy} requires the best population
tree to recover a fixed fraction of the remaining block residual.
Since
$\Sigma_{\phi,t}^{-1}\preceq\lambda^{-1}I$, it implies
\begin{align*}
    \|r_t^\psi-h_t^\psi\|_{\Sigma_{\phi,t}}^2
    \le
    (1-\gamma_\psi^2)\|r_t^\psi\|_{\Sigma_{\phi,t}}^2,\quad
    \|r_t^\phi-h_t^\phi\|_{\Sigma_{\psi,t+1}}^2
    \le
    (1-\gamma_\phi^2)\|r_t^\phi\|_{\Sigma_{\psi,t+1}}^2.
\end{align*}
This $L^2$ formulation is slightly stronger than required for the population
descent result, for which the two displayed direction inequalities suffice.
The condition holds with $\gamma_\phi=\gamma_\psi=1$ when the unrestricted
conditional means are recovered exactly. It is analogous to the
relative-accuracy or weak-learning conditions used in boosting analyses
\citep{grubb2011generalized,lu2020randomized,lu2020accelerating} and to the
Hessian-induced condition used for restricted Newton boosting by
\citet{zozoulenko2026gradient}.}

\begin{theorem}
\label{thm:weak_learner_convergence}
Suppose the population tree regressions attain their minima,
Assumptions~\ref{ass:nondegenerate_factors}
and~\ref{ass:relative_tree_accuracy} hold, and let
$\eta_\phi,\eta_\psi\in(0,1]$. Consider the approximate updates
\begin{align*}
    { \psi_{t+1}}
    &=\psi_t+\eta_\psi h_t^\psi
      =(1-\eta_\psi)\psi_t
        +\eta_\psi\psi_{t+1}^{\mathcal T,*},\\
    { \phi_{t+1}}
    &=\phi_t+\eta_\phi h_t^\phi
      =(1-\eta_\phi)\phi_t
        +\eta_\phi\phi_{t+1}^{\mathcal T,*},
\end{align*}
and define
\[
    q_\psi
    \coloneqq 1-\eta_\psi
      +\eta_\psi\sqrt{1-\gamma_\psi^2},
    \quad
    q_\phi
    \coloneqq 1-\eta_\phi
      +\eta_\phi\sqrt{1-\gamma_\phi^2},
    \quad
    a\coloneqq\min\{1-q_\phi^2,1-q_\psi^2\}.
\]
Then $q_\phi,q_\psi<1$, hence $a>0$, and the loss $L(\phi_t,\psi_t)$
decreases monotonically to a finite limit
$L_\infty\coloneqq\lim_{t\to\infty}L(\phi_t,\psi_t)$. For every $T\ge1$,
\[
    \min_{0\le t<T}\bigl(G_t^\phi+G_t^\psi\bigr)
    \le
    \frac{L(\phi_0,\psi_0)-L_\infty}{aT},
\]
and
\[
    \min_{0\le t<T}
    \Bigl(
      \|\nabla_\phi L(\phi_t,\psi_{t+1})\|_{L^2(P_X)}^2
      +
      \|\nabla_\psi L(\phi_t,\psi_t)\|_{L^2(P_Y)}^2
    \Bigr)
    \le
    \frac{4\Lambda\bigl(L(\phi_0,\psi_0)-L_\infty\bigr)}{aT}.
\]
Moreover, the per-iteration descent inequality underlying these bounds,
together with convergence of $L_t\coloneqq L(\phi_t,\psi_t)$, yields
full-sequence convergence: $G_t^\phi,G_t^\psi\to0$. By the gradient bound of
Lemma~\ref{lem:gap_characterization}, the corresponding functional
gradients $\nabla_\psi L(\phi_t,\psi_t)$ and
$\nabla_\phi L(\phi_t,\psi_{t+1})$ also vanish. If the iterate sequence
converges in product $L^2$, its limit is stationary. Exact block
regression is recovered by taking $h_t^\phi=r_t^\phi$, $h_t^\psi=r_t^\psi$,
and $\gamma_\phi=\gamma_\psi=1$, giving
$a=\min\{\eta_\phi(2-\eta_\phi),\eta_\psi(2-\eta_\psi)\}$.
\end{theorem}
}


\subsection{Population landscape}

{ It remains to connect stationarity of the factorized objective with the
spectral target from Section~\ref{sec:background}. Recall the centered
operator $\CEO_0:L_0^2(P_Y)\to L_0^2(P_X)$, whose kernel is $\kp_0$. Its
adjoint satisfies
\[
    \CEO_0^*f(y)=\E{f(X)\mid Y=y},
    \qquad f\in L_0^2(P_X).
\]
}

\begin{theorem}
\label{thm:population_landscape}
Let $(\phi,\psi)$ be a nondegenerate, balanced stationary point of the
population objective over the full centered $L^2$ spaces, with
\[
    \Sigma_\phi=\Sigma_\psi
    =D=\operatorname{diag}(s_1,\ldots,s_d)\succ0.
\]
Set $\phi^\perp=D^{-1/2}\phi$ and
$\psi^\perp=D^{-1/2}\psi$. Then, for every $i$,
\[
    \CEO_0\psi_i^\perp=s_i\phi_i^\perp,
    \quad
    \CEO_0^*\phi_i^\perp=s_i\psi_i^\perp.
\]
If $(\phi,\psi)$ is a local minimum, then, for a suitable leading singular
system,
\[
    \operatorname{span}
    \{\phi_1^\perp,\ldots,\phi_d^\perp\}
    =
    \operatorname{span}
    \{\phi_1^\star,\ldots,\phi_d^\star\},
\]
with the analogous equality holding for the learned and leading right
singular subspaces.
Consequently, every such local minimum is global and realizes the optimal
rank-$d$ truncation.
\end{theorem}
It follows that a local minimum cannot omit a singular direction whose
singular value is strictly larger than one of the selected values. Its
components must therefore span a leading $d$ singular subspace. By the
Eckart--Young--Mirsky theorem, stated as
Theorem~\ref{thm:optimal_truncation}, the corresponding rank-$d$
approximation is globally optimal. Therefore, within the nondegenerate
balanced population landscape, every local minimum is global. Finally, when
$\sigma_d>\sigma_{d+1}$, no boundary tie is possible, so the leading left
and right singular subspaces are uniquely determined, up to rotations
within repeated singular-value subspaces.


\subsection{Finite-sample guarantee}
\label{sec:finite_sample_guarantee}

We give a qualitative guarantee for the final fitted model rather than track
the full empirical optimization path. Fix an iteration budget, and let
$\mathcal H_X\subset L_0^2(P_X)$ and
$\mathcal H_Y\subset L_0^2(P_Y)$ be fixed measurable scalar function classes,
chosen independently of the sample and uniformly bounded by $B$, that contain
the initial coordinates and all possible scalar block responses of the
procedure under consideration. Set
\[
    \mathcal F_X\coloneqq\operatorname{conv}(\mathcal H_X),
    \qquad
    \mathcal F_Y\coloneqq\operatorname{conv}(\mathcal H_Y).
\]
Because every damped update is a convex combination of the current factor
and a block response, the coordinates of the factors before final balancing
belong to these classes. We use this containment as the only assumption on
the finite-sample optimization procedure.

For a scalar class $\mathcal H$ on
$\mathcal Z$ and a sample $S=(z_1,\ldots,z_n)$, define its empirical and
worst-case Rademacher complexities by
\[
    \widehat{\mathfrak R}_S(\mathcal H)
    \coloneqq
    \mathbb E_{\epsilon}\!\left[\sup_{h\in\mathcal H}
      \frac1n\sum_{i=1}^n\epsilon_i h(z_i)
    \right]
    ,\qquad
    \mathfrak R_n(\mathcal H)
    \coloneqq
    \sup_{S\in\mathcal Z^n}\widehat{\mathfrak R}_S(\mathcal H),
\]
where the $\epsilon_i$ are independent and uniform on $\{-1,1\}$; see
\citet{bartlett2002rademacher}. Define
for $\delta\in(0,1)$
\[
    \Delta_n(\delta)
    \coloneqq
    \mathfrak R_n(\mathcal H_X)+\mathfrak R_n(\mathcal H_Y)
      +\sqrt{\frac{\log(8d^2/\delta)}{n}}.
\]

Following the definitions in Section \ref{sec:practical_implementation}, define the empirical objective
\[
    \widehat L_n(\phi,\psi)
    \coloneqq
    \tr{\widehat\Sigma_\phi\widehat\Sigma_\psi}
      +2\overline\phi_n^\top\overline\psi_n
      -\frac2n\sum_{i=1}^n\phi(X_i)^\top\psi(Y_i).
\]
The product-of-means term is the empirical counterpart of the centering
correction in the squared error for $\kp_0$. Its population value vanishes
because every function in $\mathcal F_X\cup\mathcal F_Y$ is centered.

\begin{theorem}
\label{thm:finite_sample_generalization}
Suppose $\mathcal D_n$ is i.i.d. Then there is a constant $C_{d,B}>0$
depending only on $d$ and $B$ such that, for every $\delta\in(0,1)$, with
probability at least $1-\delta$,
\[
    \sup_{\substack{\phi\in\mathcal F_X^d\\
                    \psi\in\mathcal F_Y^d}}
    |\widehat L_n(\phi,\psi)-L(\phi,\psi)|
    \le C_{d,B}\Delta_n(\delta).
\]
\end{theorem}

Thus empirical and population loss are uniformly close over the fitted
class, and an approximate empirical minimizer competes with the best
population pair in that class up to estimation and optimization error. In
particular, the uniform gap vanishes whenever
$\mathfrak R_n(\mathcal H_X)+\mathfrak R_n(\mathcal H_Y)\to0$. Applying the
result before any well-defined final balancing also controls the returned
model, because balancing preserves the represented kernel and hence both
losses. A proof is given in Appendix~\ref{app:finite_sample_proofs}.

\section{Experiments}
\label{sec:experiments}

We evaluate \ourmethod\ on controlled synthetic problems for which the
density ratio and its singular structure are known. Full training details
are provided in Appendix~\ref{app:experiments}.

\subsection{Recovering the density ratio kernel on synthetic problems}
\label{sec:kernel_recovery}

We first ask whether \ourmethod\ recovers a known low-rank density ratio
kernel, under three constructions that each stress a different aspect of
the estimation problem: closely spaced singular values that require
recovering a shared subspace rather than individual functions
(Section~\ref{sec:rank_three_experiment}), a discontinuous
piecewise-constant kernel that tests the adaptive partitioning of the tree
weak learners (Section~\ref{sec:discontinuous_regions}), and a
high-dimensional tabular input with irrelevant coordinates that tests
automatic feature selection (Section~\ref{sec:tabular_irrelevant_features}).
Each figure illustrates a single representative fit, following the
same train--validation protocol; the baseline comparison tables in each of
the first three subsections instead report the mean and $95\%$ confidence
interval over $10$ independent data draws, to assess whether the
differences between methods are distinguishable from sampling
variability.

Besides validation loss and kernel RMSE (the RMSE of $\widehat{\kp}$), we
report the Euclidean error of the singular values,
$\lVert\widehat{\boldsymbol{\sigma}}-\boldsymbol{\sigma}\rVert_2$.
Subspace error is the average normalized Frobenius distance between the
estimated and exact left and right projection matrices on the evaluation
grid. Orthogonality error is the average normalized Frobenius distance
between the empirical Gram matrices of the left and right singular
functions and the identity.

\subsubsection{Close singular values (rank-three)}
\label{sec:rank_three_experiment}

We consider a setting in which the leading singular subspace, rather
than a single singular function, must be recovered. The marginals remain
uniform on $[-1,1]$, and we set
\[
    e_1(t)=\sqrt{2}\sin(\pi t),
    \qquad
    e_2(t)=\sqrt{2}\cos(\pi t),
    \qquad
    e_3(t)=\sqrt{2}\sin(2\pi t),
\]
and
\[
    \kp(x,y)
    =
    1+\sum_{j=1}^{3}\sigma_j e_j(x)e_j(y),
    \qquad
    (\sigma_1,\sigma_2,\sigma_3)=(0.18,0.16,0.12).
\]
The functions $\{e_j\}_{j=1}^3$ are orthonormal under both marginals.
Moreover,
$\kp(x,y)\geq 1-2\sum_{j=1}^3\sigma_j=0.08$, so this construction defines
a valid density ratio. The close singular values make the individual
singular functions sensitive to perturbations and motivate evaluating the
three-dimensional singular subspaces.

\begin{figure}[htbp]
    \centering
    \includegraphics[width=\linewidth]
        {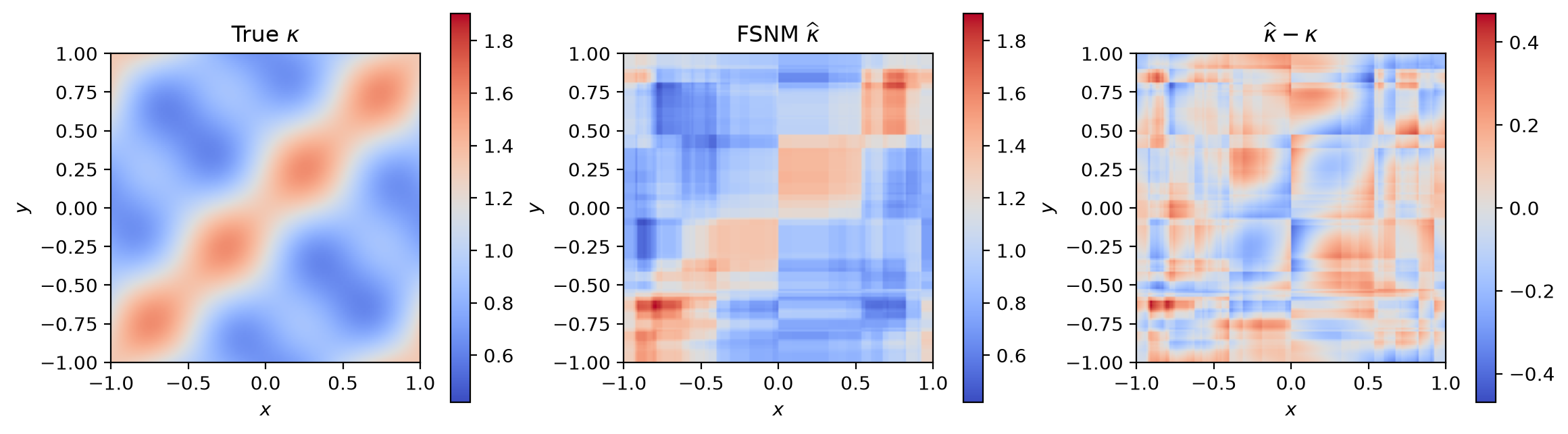}
    \caption{Rank-three synthetic experiment with close singular values.
    From left to right: the exact density ratio $\kp$, the \ourmethod\
    estimate $\widehat{\kp}$, and the signed error
    $\widehat{\kp}-\kp$.}
    \label{fig:rank_three_estimate}
\end{figure}

We draw candidate pairs from $P_X\otimes P_Y$ and accept $(x,y)$ with
probability
\[
    \frac{\kp(x,y)}{1+2\sum_{j=1}^3\sigma_j}.
\]
This rejection sampler produces i.i.d.\ observations from $P_{X,Y}$. We fit
a rank-$d=3$ model and evaluate it on a $160\times160$ uniform grid over
$[-1,1]^2$. The training details are reported in
Appendix~\ref{app:experiments}.

This single representative fit has validation loss $-0.0585$ and
orthogonality error $0.0209$; the multi-seed comparison against three
baselines is reported separately in Table~\ref{tab:rank3_baselines}
below. Its estimated spectrum is
$(0.1690,0.1386,0.1120)$, close to the population spectrum
$(0.18,0.16,0.12)$. Figure~\ref{fig:rank_three_estimate} shows that the
estimate recovers the interaction structure of the exact density ratio,
while retaining the piecewise-constant structure of the tree learners.
Figure~\ref{fig:rank_three_loss} shows the training and validation losses
over all $40$ iterations and marks the selected iteration.

\begin{figure}[htbp]
    \centering
    \includegraphics[width=0.55\linewidth]
        {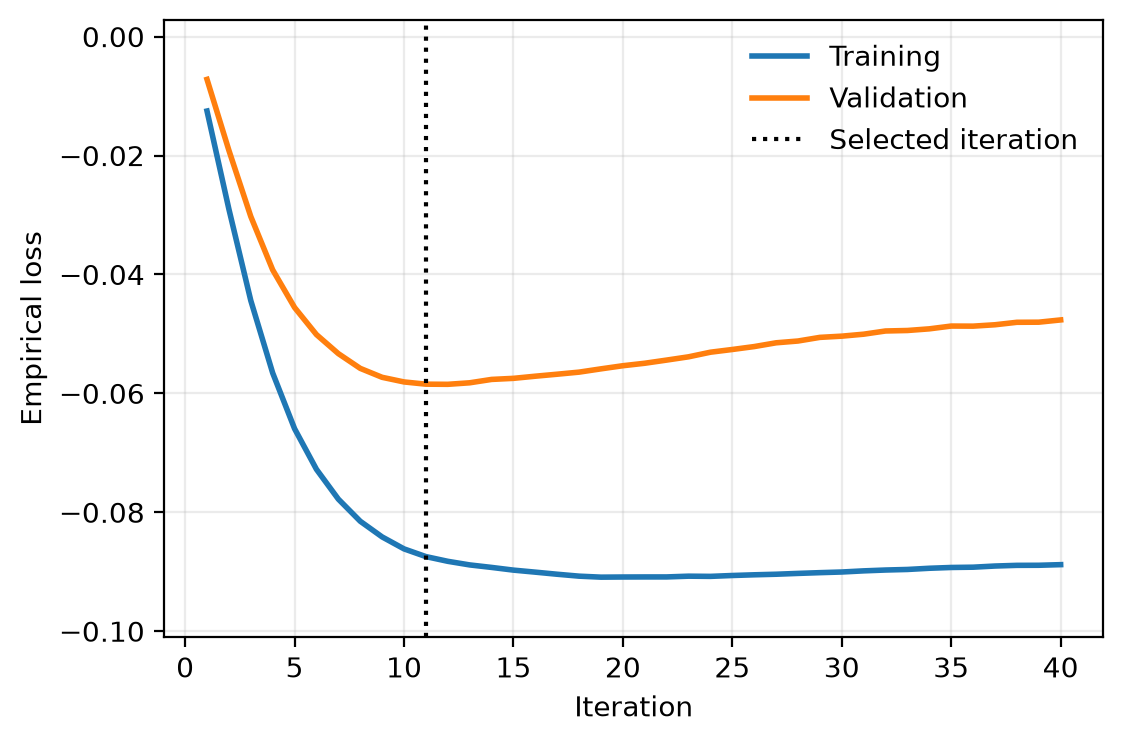}
    \caption{Empirical training and validation losses over $40$ boosting
    iterations. The vertical line marks the iteration selected by validation.}
    \label{fig:rank_three_loss}
\end{figure}

As a point of comparison, we fit three additional baselines on the
same training and validation samples: ACE \citep{breiman1985estimating};
uLSIF \citep{kanamori2009least}, which targets the full density ratio
$\kp$ pointwise without a low-rank factorization; and regularized kernel
CCA \citep{bach2002kernel}. Hyperparameters not already fixed by the
selected \ourmethod\ configuration are chosen for each baseline by its own
analogous validation criterion; complete implementation details are given
in Appendix~\ref{app:experiments}.
Table~\ref{tab:rank3_baselines} reports results over the $10$
replicates described above.
\ourmethod\ and uLSIF are statistically tied for the lowest kernel RMSE;
\ourmethod\ and ACE are tied for the lowest spectrum error; and \ourmethod,
ACE, and Kernel CCA are all tied for the lowest subspace error, which
varies more across replicates than the other two metrics. uLSIF does not
expose a spectral decomposition, so it has no spectrum or subspace error.

\begin{table}[htbp]
    \centering
    \caption{Baseline comparison on the rank-three synthetic experiment,
    over $10$ independent data draws (mean $\pm$ $95\%$ CI). Bold entries
    are statistically tied for best in that column, i.e., their interval
    overlaps the best mean's interval. uLSIF does not factorize
    $\widehat\kp$, so it has no associated spectrum or subspace error.}
    \label{tab:rank3_baselines}
    \small
    \begin{tabular}{lccc}
        \toprule
        Method & Kernel RMSE & Spectrum error & Subspace error \\
        \midrule
        \ourmethod  & $\mathbf{0.1261 \pm 0.0075}$ & $\mathbf{0.0286 \pm 0.0058}$ & $\mathbf{0.3579 \pm 0.0242}$ \\
        ACE         & $0.1489 \pm 0.0046$          & $\mathbf{0.0238 \pm 0.0048}$ & $\mathbf{0.4010 \pm 0.0322}$ \\
        uLSIF       & $\mathbf{0.1149 \pm 0.0044}$ & --                            & --                            \\
        Kernel CCA  & $0.1756 \pm 0.0198$           & $0.0720 \pm 0.0196$          & $\mathbf{0.3613 \pm 0.0663}$ \\
        \bottomrule
    \end{tabular}
\end{table}

\subsubsection{Discontinuous regional structure}
\label{sec:discontinuous_regions}

This setting tests whether the adaptive splits of the tree weak
learners can discover a piecewise-constant kernel on their own. Let both
marginals be uniform on $[-1,1]$ and partition this interval at
$-0.5$, $0$, and $0.5$. For the four resulting regions, define the three
contrasts as the rows of
\[
    H=
    \begin{pmatrix}
        1& 1& 1\\
        1&-1&-1\\
       -1& 1&-1\\
       -1&-1& 1
    \end{pmatrix}.
\]
If $r(t)$ is the region containing $t$, set
$h(t)=H_{r(t),:}$. The coordinates of $h$ are centered and orthonormal under
the uniform marginal. We use the piecewise-constant density ratio
\[
    \kp(x,y)
    =1+\sum_{j=1}^3\sigma_j h_j(x)h_j(y),
    \qquad
    (\sigma_1,\sigma_2,\sigma_3)=(0.35,0.25,0.15).
\]
Since $\kp\geq1-\sum_j\sigma_j=0.25$, this defines a valid joint
distribution with an exactly known rank-three centered kernel. We
use the same train--validation protocol as
Section~\ref{sec:rank_three_experiment}, with implementation details in
Appendix~\ref{app:experiments}.

Validation selects iteration $15$. The estimated spectrum of this
representative fit is
$(0.3514,0.2589,0.1667)$; the multi-seed comparison against three
baselines is reported separately in Table~\ref{tab:region_baselines}
below. Figure~\ref{fig:discontinuous_regions} shows that the fitted tree expansion
recovers both the rectangular partition and the interaction pattern without
being supplied the cut locations.

\begin{figure}[htbp]
    \centering
    \includegraphics[width=\linewidth]
        {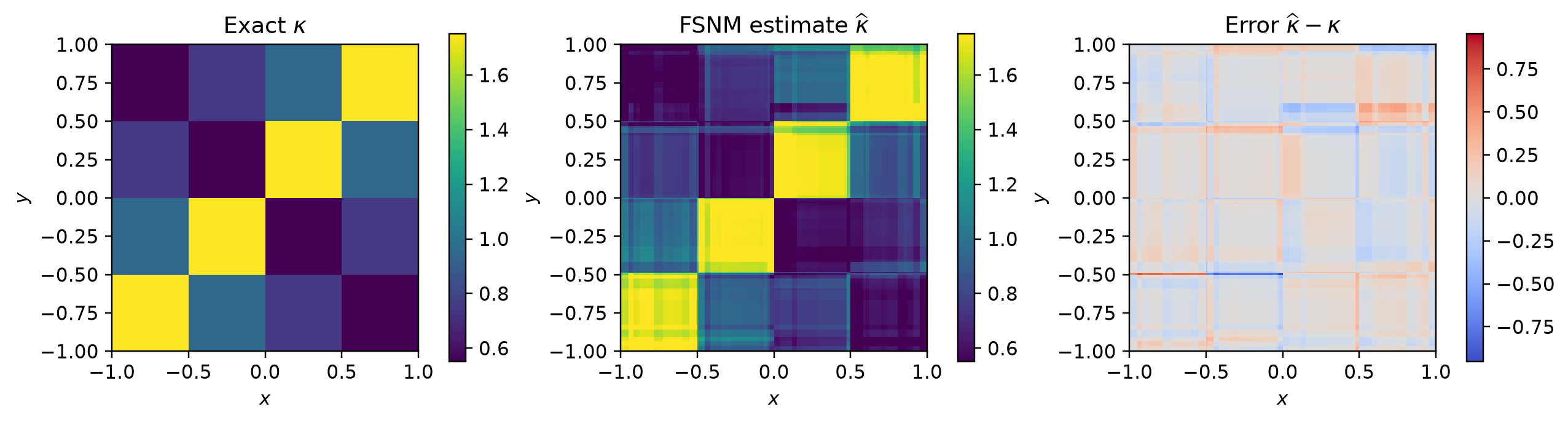}
    \caption{Discontinuous regional experiment. From left to right: the
    exact piecewise-constant density ratio, the FSNM estimate using tree weak
    learners, and their signed difference.}
    \label{fig:discontinuous_regions}
\end{figure}

We repeat the baseline comparison from
Section~\ref{sec:rank_three_experiment} on this discontinuous kernel,
matching each method's own configuration except for the tree weak-learner
budget, which uses maximum depth $3$ and minimum leaf size $120$ for both
\ourmethod\ and ACE, as in Appendix~\ref{app:experiments}.
Table~\ref{tab:region_baselines} shows that \ourmethod\ and ACE are
statistically tied for the lowest error on every metric, and both clearly
outperform the two Gaussian-kernel baselines: uLSIF and kernel CCA are not
well suited to a piecewise-constant density ratio, since their smooth
kernels cannot represent the sharp regional boundaries that the tree-based
methods recover by adaptive partitioning.

\begin{table}[htbp]
    \centering
    \small
    \caption{Baseline comparison on the discontinuous regional experiment
    (mean $\pm$ $95\%$ CI, bold = tied for best, as in
    Table~\ref{tab:rank3_baselines}).}
    \label{tab:region_baselines}
    \begin{tabular}{lccc}
        \toprule
        Method & Kernel RMSE & Spectrum error & Subspace error \\
        \midrule
        \ourmethod  & $\mathbf{0.1393 \pm 0.0198}$ & $\mathbf{0.0301 \pm 0.0113}$ & $\mathbf{0.2678 \pm 0.0846}$ \\
        ACE         & $\mathbf{0.1583 \pm 0.0152}$ & $\mathbf{0.0353 \pm 0.0066}$ & $\mathbf{0.2978 \pm 0.0742}$ \\
        uLSIF       & $0.2312 \pm 0.0065$          & --                            & --                            \\
        Kernel CCA  & $0.2908 \pm 0.0148$           & $0.0716 \pm 0.0287$          & $0.4878 \pm 0.0317$          \\
        \bottomrule
    \end{tabular}
\end{table}

\subsubsection{Tabular inputs with irrelevant features}
\label{sec:tabular_irrelevant_features}

This setting tests whether the fitted kernel automatically ignores
input coordinates that play no role in the joint law. We next let $X$ be
uniform on $[-1,1]^{20}$ and retain a uniform scalar
marginal for $Y$. Only the first three coordinates of $X$ affect the joint
law. In particular, define
\[
    a(X)=
    \bigl(\operatorname{sign}(X_1),\operatorname{sign}(X_2),
    \operatorname{sign}(X_3)\bigr)
\]
and use the same response-side regional contrasts $h(Y)$ as above. The
density ratio is
\[
    \kp(x,y)
    =1+\sum_{j=1}^3\sigma_j a_j(x)h_j(y),
    \qquad
    (\sigma_1,\sigma_2,\sigma_3)=(0.30,0.22,0.14).
\]
Thus $X_4,\ldots,X_{20}$ are independent noise coordinates, while the exact
kernel and spectrum remain available for evaluation. Implementation
details, shared with the preceding regional experiment, are given in
Appendix~\ref{app:experiments}.

\begin{figure}[htbp]
    \centering
    \includegraphics[width=\linewidth]
        {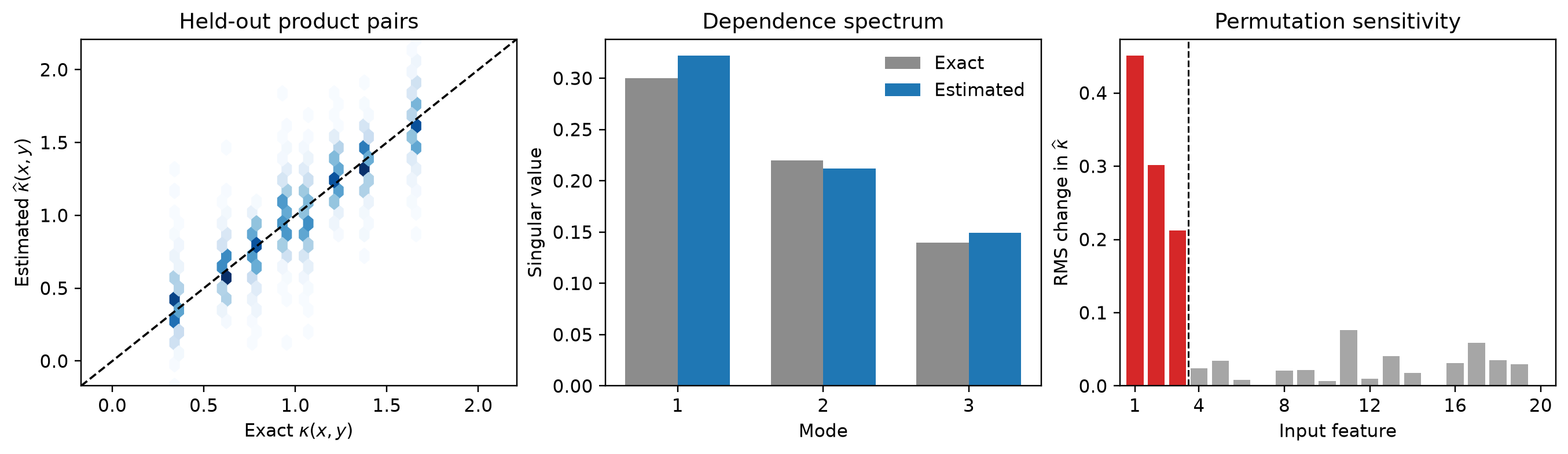}
    \caption{Twenty-dimensional tabular experiment. Left: exact and
    estimated density ratios on held-out product pairs. Center: exact and
    estimated spectra. Right: post-fit permutation sensitivity; red bars are
    the three relevant inputs and gray bars are the $17$ irrelevant inputs.}
    \label{fig:tabular_irrelevant_features}
\end{figure}

Validation selects iteration $32$. On $10{,}000$ independent pairs from the
product of the marginals, this representative fit has estimated spectrum
$(0.3220,0.2122,0.1491)$; the multi-seed comparison against three
baselines is reported separately in Table~\ref{tab:tabular_baselines}
below. We also compute a post-fit permutation sensitivity for
each input coordinate,
\[
    I_k
    =\left(
      \frac{1}{m}\sum_{i=1}^m
      \left[
        \widehat\kp(X_i,Y_i)
        -\widehat\kp(X_i^{\pi_k},Y_i)
      \right]^2
    \right)^{1/2},
\]
where $X^{\pi_k}$ independently permutes column $k$ of the evaluation input.
This diagnostic is computed after training and does not enter the fitting
objective. Its mean is $0.3214$ for the three relevant variables and
$0.0244$ for the $17$ irrelevant variables.

Figure~\ref{fig:tabular_irrelevant_features} reports the held-out kernel
estimates, recovered spectrum, and all permutation sensitivities. The three
coordinates used by the population kernel have the three largest
sensitivities, while the noise coordinates have little influence on the
fitted interaction.

Table~\ref{tab:tabular_baselines} repeats the comparison on the
twenty-dimensional tabular setting, evaluated on freshly sampled
held-out product pairs each time. The two Gaussian-kernel
baselines are substantially worse than both tree-based methods: their
kernel distance combines all twenty input coordinates without automatic
feature selection, so the seventeen irrelevant coordinates dilute the
signal from the three relevant ones. \ourmethod\ and ACE, whose weak
learners split on individual coordinates, are far less sensitive to this
irrelevant-feature noise, and are statistically tied for the lowest error
on every metric.

\begin{table}[htbp]
    \centering
    \small
    \caption{Baseline comparison on the twenty-dimensional tabular
    experiment with seventeen irrelevant input coordinates (mean $\pm$
    $95\%$ CI, bold = tied for best, as in Table~\ref{tab:rank3_baselines}).}
    \label{tab:tabular_baselines}
    \begin{tabular}{lccc}
        \toprule
        Method & Kernel RMSE & Spectrum error & Subspace error \\
        \midrule
        \ourmethod  & $\mathbf{0.1880 \pm 0.0227}$ & $\mathbf{0.0363 \pm 0.0114}$ & $\mathbf{0.4689 \pm 0.1062}$ \\
        ACE         & $\mathbf{0.1818 \pm 0.0152}$ & $\mathbf{0.0375 \pm 0.0065}$ & $\mathbf{0.3754 \pm 0.0689}$ \\
        uLSIF       & $0.7102 \pm 0.0042$          & --                            & --                            \\
        Kernel CCA  & $0.6651 \pm 0.0180$           & $0.3618 \pm 0.0117$          & $0.6994 \pm 0.0088$          \\
        \bottomrule
    \end{tabular}
\end{table}

\subsection{Querying the fitted kernel}
\label{sec:kernel_querying}

Having fit a single kernel above, we next show that it answers several
different downstream queries without refitting: conditional moments and a
tail probability (Section~\ref{sec:conditional_queries}), and the full
conditional distribution with bootstrap uncertainty bands
(Section~\ref{sec:conditional_cdf_uncertainty}). Both use the tree-based
fit from Section~\ref{sec:rank_three_experiment}.

\subsubsection{Conditional mean, variance, and tail probability}
\label{sec:conditional_queries}

As an application of the fitted kernel above, $\widehat{\kp}$ can be used to
approximate the conditional expectation of any integrable query $f$ through
\[
    \E{f(Y)\mid X=x}
    =
    \int f(y)\kp(x,y)\,\mathrm{d}P_Y(y).
\]
Using the response coordinates $\{Y_i\}_{i=1}^n$ of the same paired training
sample used to fit $\widehat{\kp}$, we use
\[
    \widehat m_f(x)
    =
    \frac{1}{n}\sum_{i=1}^n
    f(Y_i)\widehat{\kp}(x,Y_i).
\]
No additional model is fitted for a new $f$. We evaluate the conditional mean
using $f(y)=y$, the conditional variance obtained from the queries $y$ and
$y^2$, and the conditional tail probability obtained from
$f(y)=\Ind{y>0.5}$. Their RMSEs over the evaluation grid are,
respectively, $0.0237$, $0.0166$, and $0.0233$.
Figure~\ref{fig:conditional_queries} shows that the three functionals are
recovered by the same learned kernel.

\begin{figure}[htbp]
    \centering
    \includegraphics[width=\linewidth]
        {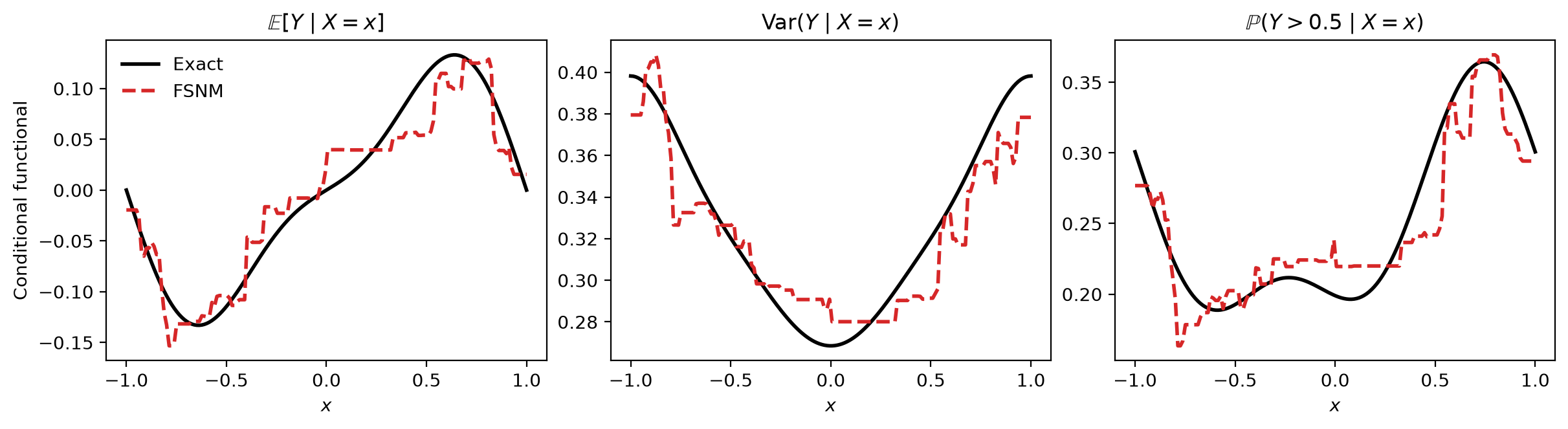}
    \caption{Conditional queries obtained from the learned rank-three kernel.
    From left to right: conditional mean, conditional variance, and conditional
    tail probability. Solid black curves are exact and dashed red curves are
    computed from $\widehat{\kp}$.}
    \label{fig:conditional_queries}
\end{figure}

\subsubsection{Conditional CDF and bootstrap uncertainty}
\label{sec:conditional_cdf_uncertainty}

The conditional CDF collects a continuum of queries from the learned kernel:
for $t\in[-1,1]$,
\[
    F_x(t)
    \coloneqq
    \Pr{Y\leq t\mid X=x}
    =
    \int \Ind{y\leq t}\kp(x,y)\,\mathrm dP_Y(y).
\]
For fixed $x$, using the same training responses as above, we first
evaluate the indicator query through the direct empirical operator.


No individual kernel weight is truncated or renormalized. Because the
finite-sample signed curve $t\mapsto\widetilde F_x(t)$ need not be monotone or
remain in $[0,1]$, we project the complete curve by isotonic regression onto
the class of valid CDFs, anchored at zero and one. This downstream projection
enforces distributional coherence without changing the learned kernel or its
direct action on an individual query.

We evaluate the CDF at $x\in\{-0.6,0,0.6\}$ and quantify uncertainty
with $100$ nonparametric bootstrap replicates. Each replicate resamples the
paired observations, refits the kernel, evaluates the direct indicator
queries, and projects the resulting complete curves, following the protocol
in Appendix~\ref{app:experiments}.

\begin{figure}[htbp]
    \centering
    \includegraphics[width=\linewidth]
        {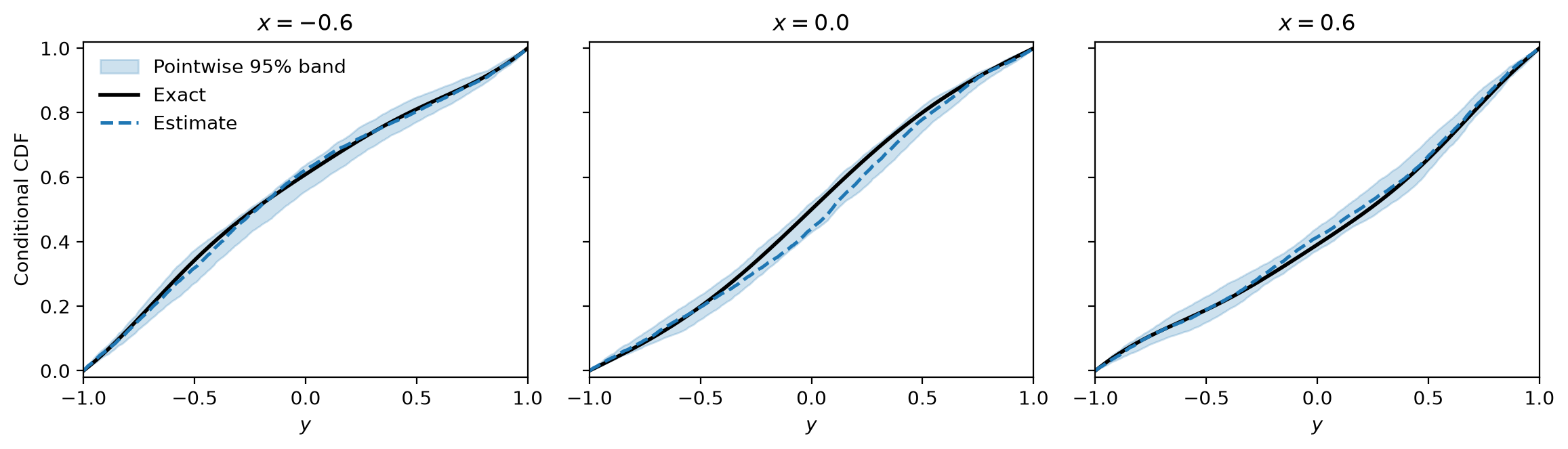}
    \caption{Projected direct conditional CDF estimates and pointwise
    bootstrap bands at three
    values of $x$ using tree weak learners. Solid black curves are exact,
    dashed blue curves are the fitted estimates, and shaded regions are
    pointwise $95\%$ bands.}
    \label{fig:tree_conditional_cdf_uncertainty}
\end{figure}

Figure~\ref{fig:tree_conditional_cdf_uncertainty} shows that the estimated
conditional distributions track the exact CDFs at all three values of $x$.
The bands have average width $0.0655$, and their gridwise coverage of the exact
curves is at least $99.7\%$ at each displayed value of $x$. Thus the bootstrap
captures the joint sampling variability of the learned tree kernel, its direct
action on the indicator queries, and the subsequent CDF projection.

\subsection{Detecting dependence beyond linear correlation}
\label{sec:dependence_detection}

As a further application of the learned kernel, we use it to distinguish independence from
dependence. The two examples have the same marginals: $X$ is uniform on
$[-1,1]$ and $Y$ has the distribution of $U^2+\epsilon$, where
$U\sim\operatorname{Unif}[-1,1]$ and
$\epsilon\sim\mathcal N(0,0.08^2)$. Only the coupling between $X$ and $Y$
changes. In both cases, we fit a rank-three model using shallow regression
trees as weak learners, with the number of boosting iterations selected by
validation. Alongside the estimated kernel and spectrum, we report a held-out
permutation test based on the test-sample dependence score. Its null
distribution is obtained by permuting the fitted $Y$-factors across the test
observations. Further details are given in
Appendix~\ref{app:experiments}.

\subsubsection{Independent variables}
\label{sec:independent_variables}

Let $X$ and $U$ be independent and set $Y=U^2+\epsilon$. Then
$X\independent Y$, so $\kp=1$ and the centered kernel $\kp_0=\kp-1$ vanishes.
Figure~\ref{fig:independent_case} shows that the learned
$\widehat\kp$ concentrates around one on independent reference pairs, while
its singular values concentrate near zero. The estimated spectrum is
$(0.0114,0.0101,0.0082)$, the grid RMS of $\widehat\kp_0$ is $0.0209$, and
the permutation test does not reject independence ($p=0.965$).

\begin{figure}[htbp]
    \centering
    \includegraphics[width=0.82\linewidth]
        {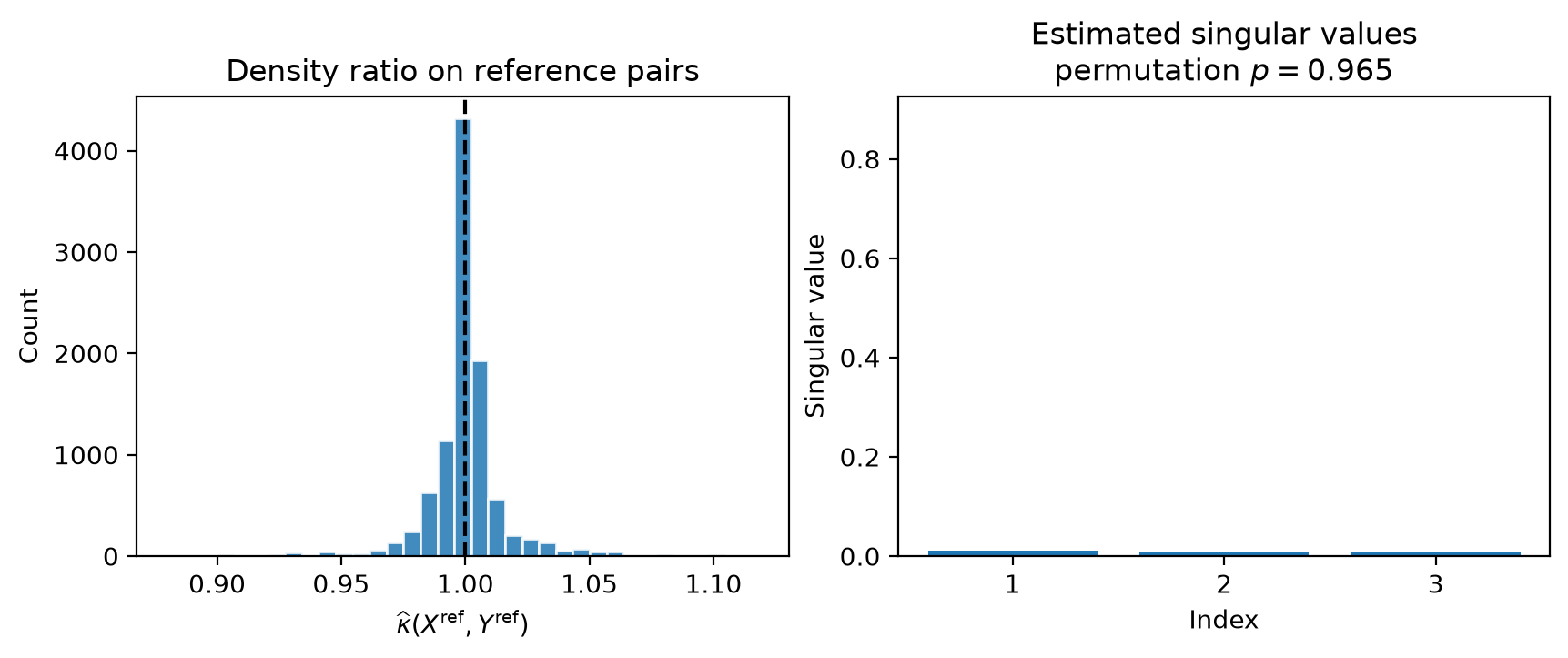}
    \caption{Independent case. Left: distribution of $\widehat\kp$ over
    independent reference pairs, with the dashed line marking the population
    value one. Right: estimated singular values, on the same vertical scale
    as Figure~\ref{fig:nonlinear_dependence}.}
    \label{fig:independent_case}
\end{figure}

The three diagnostics give the same conclusion. The density ratio estimates
remain close to the independence value, the spectral contribution is small,
and the permutation test finds no evidence against independence. Moreover,
validation selects only one boosting iteration. Thus, even with adaptive tree
learners, the fitted model does not introduce a material interaction when the
joint distribution is the product of its marginals.

\subsubsection{Nonlinear dependence with zero correlation}
\label{sec:nonlinear_dependence}

We now set $Y=X^2+\epsilon$. Although $X$ and $Y$ are dependent, symmetry
gives $\operatorname{Cov}(X,Y)=0$. In the test sample, Pearson's correlation
is $-0.0191$, while the permutation test detects dependence ($p=0.001$).
Figure~\ref{fig:nonlinear_dependence} shows the nonlinear structure recovered
by $\widehat\kp_0$ and the clearly nonzero estimated spectrum
$(0.8579,0.7659,0.5186)$.

\begin{figure}[htbp]
    \centering
    \includegraphics[width=\linewidth]
        {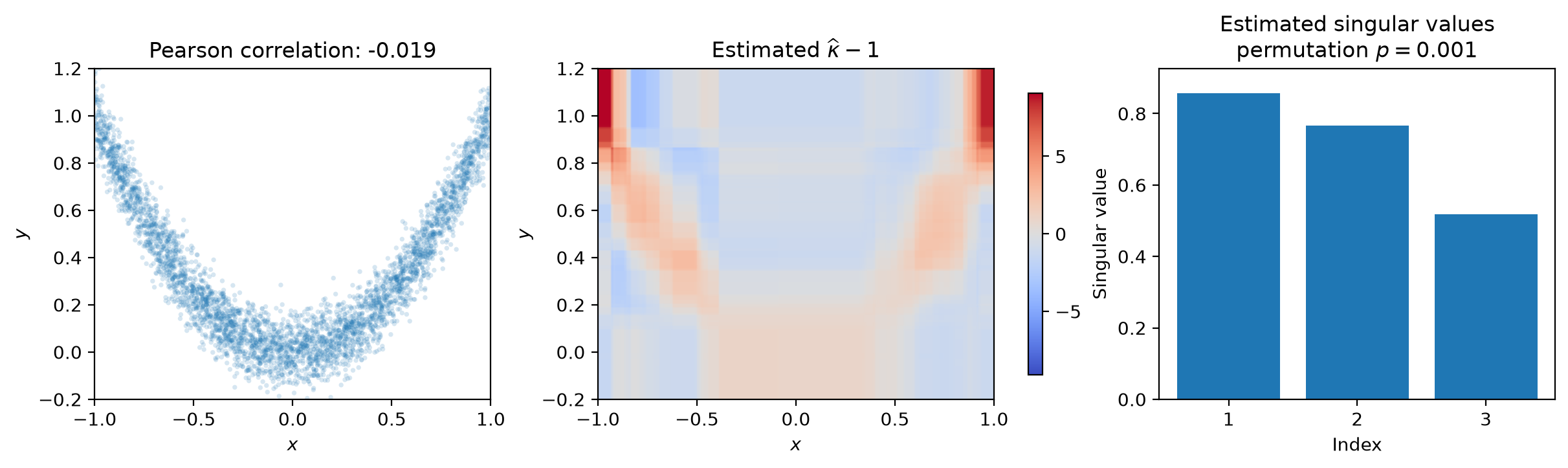}
    \caption{Nonlinear dependent case $Y=X^2+\epsilon$. From left to right:
    test observations, learned centered density ratio kernel, and estimated
    singular values. Pearson's correlation is near zero, whereas the learned
    kernel and the held-out permutation test detect dependence.}
    \label{fig:nonlinear_dependence}
\end{figure}

Here the conclusion is reversed under the same marginals and tree class. The
symmetric U-shaped relation makes linear covariance vanish, explaining why
Pearson's correlation is uninformative. The learned kernel nevertheless
retains the nonlinear interaction, and its large singular values agree with
the permutation rejection. This example shows that the method detects a
general departure of the density ratio from one rather than only linear
association.

\subsection{\texorpdfstring{Real-data application: glass
composition--property data}{Real-data application: glass composition--property data}}
\label{sec:glass_experiment}

Predicting how a glass's physical and optical properties depend on its
oxide composition is central to glass discovery, since synthesizing and
characterizing a candidate composition is slow and costly
\citep{mazurin2005glass,thomaello2025patents,cassar2021predicting}.
Tree ensembles are particularly effective for sparse tabular composition
data \citep{grinsztajn2022tree,cassar2021predicting}. \ourmethod\ retains
this inductive bias while learning a reusable low-rank conditional kernel
rather than a single scalar predictor.

We use composition--refractive-index (RI) records integrated from
SciGlass \citep{mazurin2005glass}, INTERGLAD \citep{interglad}, and patent
data \citep{thomaello2025patents}. To reduce chemical heterogeneity, we
restrict attention to glasses whose nonzero composition is contained in the
ten-oxide family in Table~\ref{tab:glass_oxide_family}; every oxide outside
this set must have recorded molar fraction zero. The family spans classical
network formers, alkali modifiers, and components used in high-index optical
glass systems \citep{zachariasen1932atomic,mao2015optical,ma2015tantalum}.
Individual glasses need not contain all ten components.

\begin{table}[htbp]
    \centering
    \footnotesize
    \caption{Ten-oxide family retained for the real-data experiment.}
    \label{tab:glass_oxide_family}
    \begin{tabular}{lll}
        \toprule
        Group & Oxides & Motivation \\
        \midrule
        Network formers & SiO$_2$, B$_2$O$_3$, P$_2$O$_5$
          & Backbone-forming oxides \\
        Alkali modifiers & Li$_2$O, Na$_2$O, K$_2$O
          & Network modification and fluxing \\
        Optical-property components &
          TiO$_2$, Nb$_2$O$_5$, Ta$_2$O$_5$, La$_2$O$_3$
          & High-index optical-glass systems \\
        \bottomrule
    \end{tabular}
\end{table}

\subsubsection{\texorpdfstring{Learning the
composition--refractive-index kernel}{Learning the composition--refractive-index kernel}}
\label{sec:glass_kernel_fit}

After retaining RI in $[1,4.5]$, enforcing the zero restriction above,
and renormalizing the retained oxide fractions to sum to one, $3{,}013$
records remain. Their RI values range from $1.348$ to $2.480$. A fixed split
(seed $2026$) assigns $2{,}109/452/452$ observations to training,
validation, and test partitions. We select among ranks $\{5,10,20\}$,
maximum tree depths $\{3,5\}$, and minimum leaf sizes $\{25,50,100\}$, using
step size $0.1$, at most $80$ iterations, and patience $10$. Validation loss
selects rank $20$, depth $5$, leaf size $25$, and iteration $39$
(validation loss $-5.6476$). Selection by validation CRPS gives the same
configuration.

Figure~\ref{fig:glass_kernel_fit} shows the selected trajectory and
spectrum. The captured energy $\sum_j\widehat\sigma_j^2$ is $4.5961$, and
the effective spectral rank is $11.41$. The first mode accounts for $19.1\%$
of the captured energy, the first ten account for $78.8\%$, and $14$ modes
are required to exceed $90\%$. Thus the restricted chemical family retains
a genuinely multivariate dependence structure, while concentrating more of
its energy in the leading modes than the heterogeneous full table.

\begin{figure}[htbp]
    \centering
    \includegraphics[width=\linewidth]
        {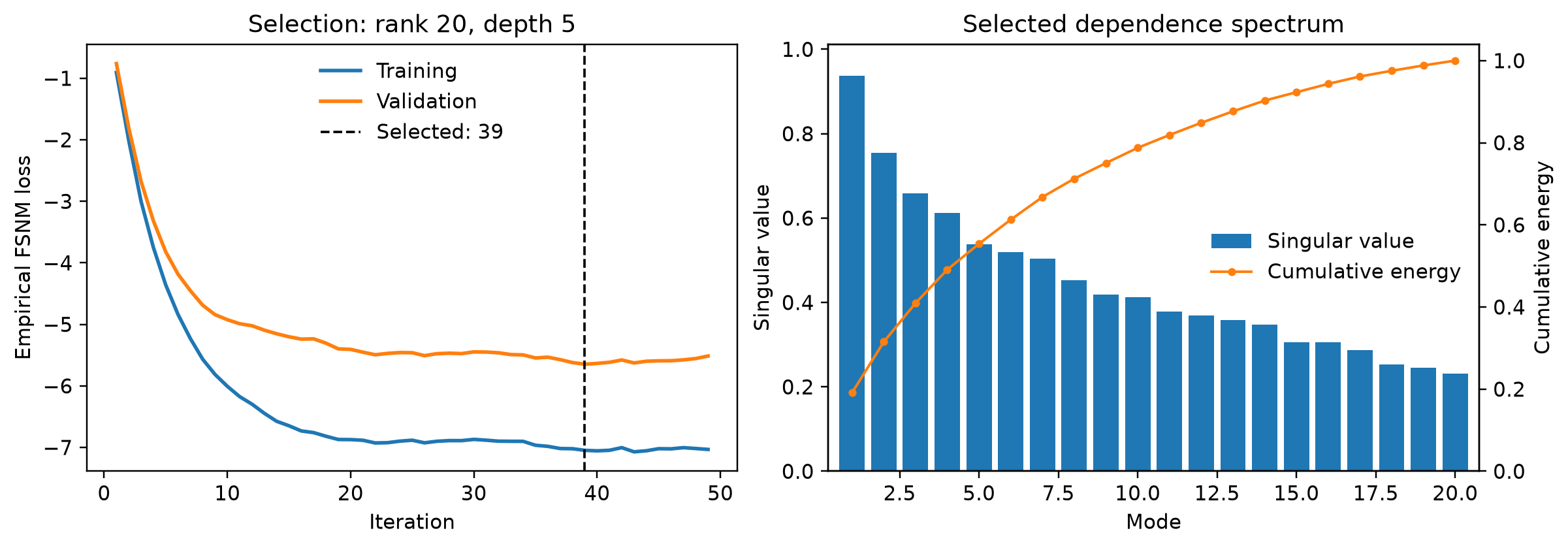}
    \caption{Fit of the restricted glass composition--RI. Left:
    training and validation losses for the selected rank-$20$ trajectory;
    early stopping chooses iteration $39$. Right: fitted singular values and
    cumulative share of $\sum_j\widehat\sigma_j^2$.}
    \label{fig:glass_kernel_fit}
\end{figure}

\subsubsection{RI prediction with the identity functional}
\label{sec:glass_identity_prediction}

We apply the plug-in estimator from
Section~\ref{sec:conditional_queries}, using the $M$ training responses as
the empirical marginal sample and taking the identity query $g(y)=y$. This
requires neither refitting nor constructing a separate scalar regressor. On
the test partition, the resulting predictor
has MSE $0.0051$, RMSE $0.0711$, MAE $0.0345$, and $R^2=0.811$. The
training-mean baseline has MSE $0.0268$, so the kernel reduces test MSE by
$81.1\%$. Figure~\ref{fig:glass_identity_prediction} shows the held-out
predictions and their calibration across prediction octiles.

\begin{figure}[htbp]
    \centering
    \includegraphics[width=\linewidth]
        {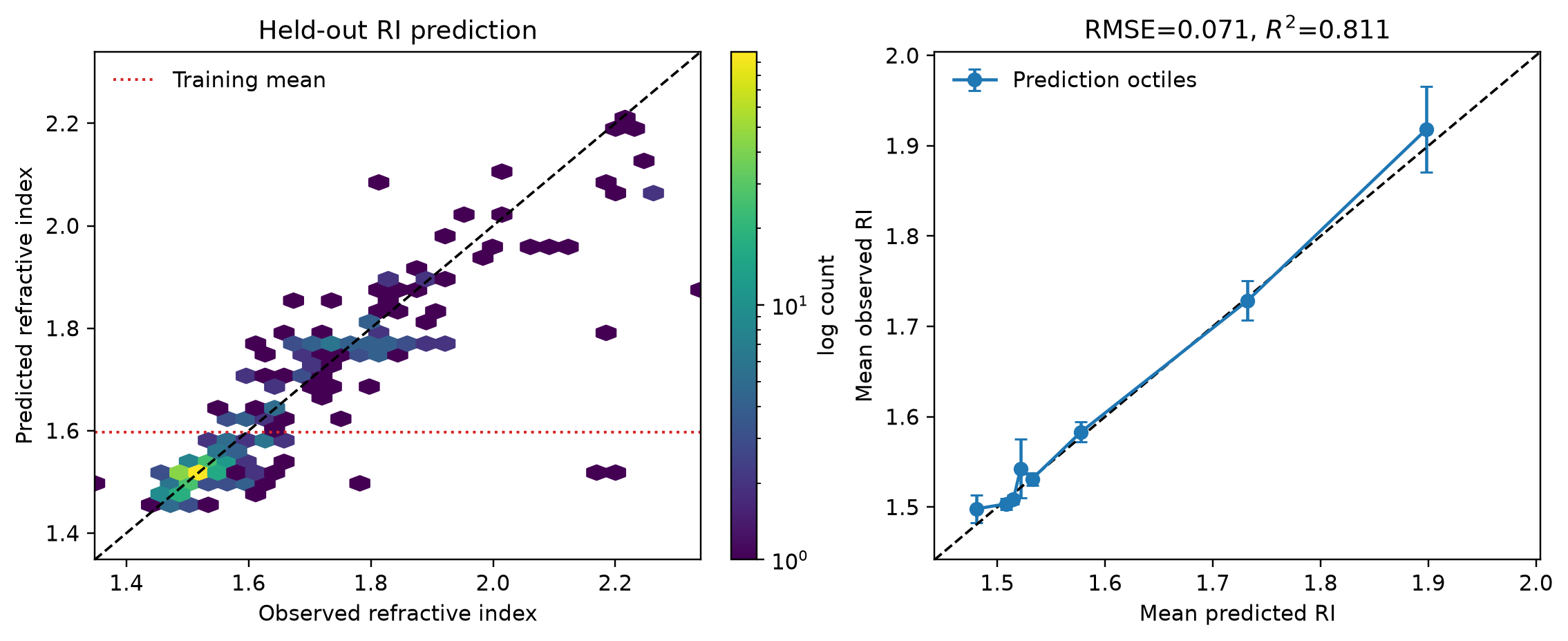}
    \caption{Prediction of RI using $g(y)=y$. Left: the dashed line is
    perfect prediction and the dotted line is the training mean. Right: mean
    observed and predicted within prediction octiles, with pointwise
    $95\%$ normal intervals for the observed means.}
    \label{fig:glass_identity_prediction}
\end{figure}

\subsubsection{Conditional intervals and screening probabilities}
\label{sec:glass_conditional_queries}

We compute the unprojected conditional CDF as in
Section~\ref{sec:conditional_cdf_uncertainty}, again using the $M$ training
responses. The screening probability at threshold $\tau$ is the plug-in
estimator from Section~\ref{sec:conditional_queries} applied to
$g(y)=\Ind{y>\tau}$; denote it by $\widetilde p_\tau(x)$. These estimates
apply the learned operator directly to the indicator query, without
truncating individual kernel weights. Because a finite-rank estimate may
nevertheless produce a nonmonotone signed
$\widetilde F_x$, we project the complete curve onto the set of CDFs by
isotonic regression. We then calibrate its quantile levels once using the
validation probability-integral-transform values and invert the resulting
CDF at test compositions. This post-processing does not refit the kernel or
use test responses. The scalar $\widetilde p_\tau(x)$ is only clipped at the
end to $[0,1]$; we use $\tau=1.8$, near the upper decile of the restricted
family.

On the $452$ test glasses, nominal $50\%$, $70\%$, and $90\%$ central
intervals achieve coverage $54.6\%$, $71.5\%$, and $89.4\%$, with mean
widths $0.043$, $0.069$, and $0.126$ RI units. For screening, the direct
indicator query has Brier score $0.0375$, compared with $0.1156$ for the
climatological predictor at the training exceedance rate $11.2\%$, a $67.6\%$
reduction. Figure~\ref{fig:glass_conditional_queries} shows the calibrated
intervals and the reliability of the uncalibrated direct screening query.

\begin{figure}[htbp]
    \centering
    \includegraphics[width=\linewidth]
        {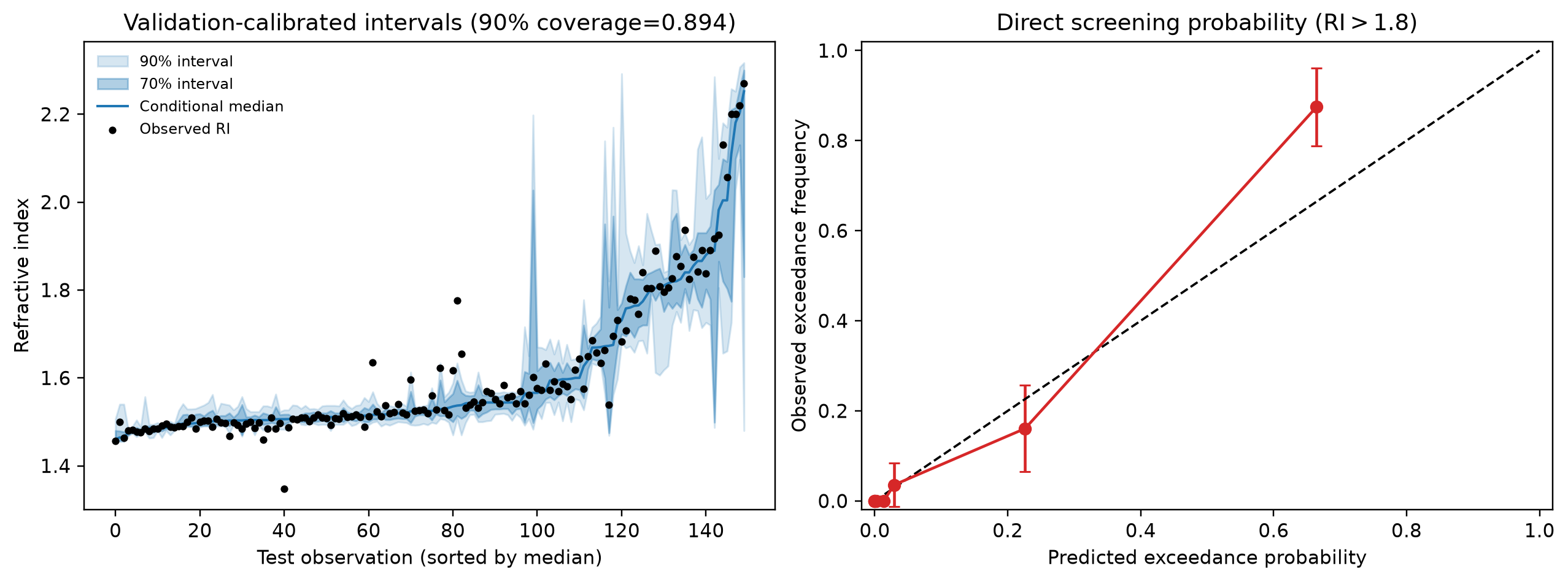}
    \caption{Conditional queries.
    Left: for
    a random subset of $150$ held-out compositions, the shaded bands are validation-calibrated $70\%$ and $90\%$
    intervals. Right: reliability of the direct indicator query
    $\widehat{\mathbb P}(\mathrm{RI}>1.8\mid x)$ in eight equal-count bins,
    with pointwise $95\%$ normal intervals for observed frequencies.}
    \label{fig:glass_conditional_queries}
\end{figure}
\section{Conclusion}
\label{sec:conclusion}

We introduced \ourmethod, which learns the leading singular structure of a
conditional expectation operator through alternating, learner-agnostic
functional block-Newton updates. We derived closed-form update directions,
established population-level descent with an $O(1/T)$ best-iterate
stationarity rate, and showed that every nondegenerate local minimum is a
globally optimal rank-$d$ approximation.

\ourmethod\ recovered synthetic density ratio kernels more accurately than
ACE, uLSIF, and kernel CCA, particularly where adaptive partitioning and
feature selection matter. On a real oxide-glass dataset, its fitted kernel
substantially produced calibrated
conditional prediction intervals and a well-discriminating direct screening
probability. This multi-query capability is shared with related
operator-learning approaches such as Neural Conditional Probability;
\ourmethod's distinguishing contribution is deriving it from a
closed-form, learner-agnostic block-Newton procedure with matching
convergence and landscape theory for tree-based weak learners.

Limitations include fixed-grid hyperparameter selection, hypothesis-class-dependent
finite-sample rates, and scalar responses. Adaptive rank selection,
structured response spaces, and the causal-inference and
reinforcement-learning applications discussed in
Section~\ref{sec:introduction} are natural extensions.

\section*{Acknowledgments}

This work was supported in part by the EU project ELIAS under grant agreement No. 101120237. Thiago was  supported by the São Paulo Research Foundation (FAPESP), grant 2026/21108-4.

\bibliographystyle{plainnat}
\bibliography{bibliography}

\appendix
\section{Proofs}
\label{app:proofs}

\subsection{Proofs for Section~\ref{sec:background}}
\label{app:proofs_model_objective}

\begin{proof}[Proof of Proposition~\ref{prop:population_loss}]
{ The squared approximation error can be written as
\[
    \|\kp_0-\kp_{\phi,\psi}\|_{P_X\otimes P_Y}^2
    =
    \|\kp_0\|_{P_X\otimes P_Y}^2
    +\|\kp_{\phi,\psi}\|_{P_X\otimes P_Y}^2
    -2\Ex{P_X\otimes P_Y}{
        \kp_0(X,Y)\kp_{\phi,\psi}(X,Y)
    }.
\]
Because $\kp_0=\kp-1$ and the factors are centered,
\begin{align*}
    \Ex{P_X\otimes P_Y}{
        \kp_0(X,Y)\kp_{\phi,\psi}(X,Y)
    }
    &=\Ex{P_X\otimes P_Y}{
        \kp(X,Y)\phi(X)^\top\psi(Y)
      }
      -\Ex{P_X}{\phi(X)}^\top\Ex{P_Y}{\psi(Y)}\\
    &=\Ex{P_{X,Y}}{\phi(X)^\top\psi(Y)}
    =A(\phi,\psi).
\end{align*}
Moreover, since $\phi(X)^\top\psi(Y)$ is scalar-valued and equals its
own transpose $\psi(Y)^\top\phi(X)$, and a scalar coincides with the trace
of the $1\times1$ matrix it forms,
\[
    \bigl(\phi(X)^\top\psi(Y)\bigr)^2
    =\bigl(\phi(X)^\top\psi(Y)\bigr)\bigl(\psi(Y)^\top\phi(X)\bigr)
    =\tr{\phi(X)^\top\bigl(\psi(Y)\psi(Y)^\top\bigr)\phi(X)}.
\]
The cyclic property of the trace moves $\phi(X)$ from the outside to the
inside of the product,
\[
    \tr{\phi(X)^\top\bigl(\psi(Y)\psi(Y)^\top\bigr)\phi(X)}
    =\tr{\bigl(\psi(Y)\psi(Y)^\top\bigr)\phi(X)\phi(X)^\top}.
\]
Taking expectations, using linearity of the trace, and the independence of
$X$ and $Y$ under $P_X\otimes P_Y$,
\begin{align*}
    C(\phi,\psi)
    &=\Ex{P_X\otimes P_Y}{\bigl(\phi(X)^\top\psi(Y)\bigr)^2}\\
    &=\tr{\Ex{P_Y}{\psi(Y)\psi(Y)^\top}\Ex{P_X}{\phi(X)\phi(X)^\top}}
    =\tr{\Sigma_\psi\Sigma_\phi}
    =\tr{\Sigma_\phi\Sigma_\psi},
\end{align*}
where the last equality applies the cyclic property of the trace once
more.
Substitution gives
\[
    \|\kp_0-\kp_{\phi,\psi}\|_{P_X\otimes P_Y}^2
    =
    \|\kp_0\|_{P_X\otimes P_Y}^2
    +C(\phi,\psi)-2A(\phi,\psi).
\]
The first term is independent of $(\phi,\psi)$, so minimizing the
approximation error is equivalent to minimizing $L=C-2A$.}
\end{proof}

\subsection{Proofs for Section~\ref{sec:method}}
\label{app:proofs_method}
\label{app:proofs_functional_calculus}

\begin{proof}[Proof of Proposition~\ref{prop:derivatives}]
For $A$, linearity in $\phi$ gives
\[
    A(\phi+\varepsilon h,\psi)
    =A(\phi,\psi)+\varepsilon\,\Ex{P_{X,Y}}{h(X)^\top\psi(Y)},
\]
so, differentiating at $\varepsilon=0$ and using the tower property to
condition on $X$,
\begin{align*}
    D_\phi A(\phi,\psi)[h]
    =\Ex{P_{X,Y}}{h(X)^\top\psi(Y)}
    =\Ex{P_X}{h(X)^\top\E{\psi(Y)\mid X}}.
\end{align*}
For $C$, expanding the square gives
\begin{align*}
    C(\phi+\varepsilon h,\psi)
    &=C(\phi,\psi)
    +2\varepsilon\,\Ex{P_X\otimes P_Y}{
        \bigl(\phi(X)^\top\psi(Y)\bigr)\bigl(h(X)^\top\psi(Y)\bigr)}\\
    &\qquad+\varepsilon^2\,\Ex{P_X\otimes P_Y}{\bigl(h(X)^\top\psi(Y)\bigr)^2},
\end{align*}
so differentiating at $\varepsilon=0$ gives
\[
    D_\phi C(\phi,\psi)[h]
    =2\,\Ex{P_X\otimes P_Y}{
        \bigl(\phi(X)^\top\psi(Y)\bigr)\bigl(h(X)^\top\psi(Y)\bigr)}.
\]
Since $\phi(X)^\top\psi(Y)$ and $h(X)^\top\psi(Y)$ are scalars,
\[
    \bigl(\phi(X)^\top\psi(Y)\bigr)\bigl(h(X)^\top\psi(Y)\bigr)
    =h(X)^\top\bigl(\psi(Y)\psi(Y)^\top\bigr)\phi(X),
\]
and taking expectations, using independence of $X$ and $Y$ under
$P_X\otimes P_Y$,
\begin{align*}
    D_\phi C(\phi,\psi)[h]
    =2\,\Ex{P_X}{h(X)^\top\Ex{P_Y}{\psi(Y)\psi(Y)^\top}\phi(X)}
    =2\Ex{P_X}{h(X)^\top\Sigma_\psi\phi(X)}.
\end{align*}
Consequently,
\[
    D_\phi L(\phi,\psi)[h]
    =
    2\Ex{P_X}{
        h(X)^\top\bigl(\Sigma_\psi\phi(X)-m_\psi(X)\bigr)}.
\]
The Riesz representation therefore yields
\[
    \nabla_\phi L(\phi,\psi)(x)
    =2\Sigma_\psi\phi(x)-2m_\psi(x).
\]
Interchanging $X,\phi$ with $Y,\psi$ gives the stated expression for
$\nabla_\psi L$.

The term $A$ is bilinear in $(\phi,\psi)$, so its diagonal second
derivatives vanish. Differentiating the first derivative of $C$ once
more in directions { $k\in L_0^2(P_X)^d$ and
$\ell\in L_0^2(P_Y)^d$} gives
\begin{align*}
    D_{\phi,\phi}C(\phi,\psi)[h,k]=2\Ex{P_X}{h(X)^\top\Sigma_\psi k(X)},
    \quad
    D_{\psi,\psi}C(\phi,\psi)[g,\ell]=2\Ex{P_Y}{g(Y)^\top\Sigma_\phi\ell(Y)}.
\end{align*}
These are also the diagonal blocks of $L$, and their associated
multiplication operators are $2\Sigma_\psi$ and $2\Sigma_\phi$,
respectively.
\end{proof}

\begin{proof}[Proof of Proposition~\ref{prop:balancing}]
\paragraph{Step 1: normalize the first factor.}
Start with
\[
    A_0=\Sigma_\phi^{-1/2},
    \qquad
    \phi^{(0)}=A_0\phi,
    \qquad
    \psi^{(0)}=A_0^{-\top}\psi
                 =\Sigma_\phi^{1/2}\psi.
\]
The inverse transformations preserve the pointwise product, and hence
$\phi^{(0)}(x)^\top\psi^{(0)}(y)=\phi(x)^\top\psi(y)$. The first
second-moment matrix is now the identity:
\[
    \Sigma_{\phi^{(0)}}=A_0\Sigma_\phi A_0^\top=I.
\]
The second one is generally not the identity; instead,
\[
    \Sigma_{\psi^{(0)}}
    =A_0^{-\top}\Sigma_\psi A_0^{-1}
    =\Sigma_\phi^{1/2}\Sigma_\psi\Sigma_\phi^{1/2}
    \eqqcolon B.
\]
Thus, this first transformation normalizes the $\phi$ factor while moving
all remaining scale and dependence between coordinates into $B$.

\paragraph{Step 2: diagonalize the second factor.}
Because $B$ is symmetric positive definite, write
\[
    B=U\Lambda U^\top,
    \qquad
    \Lambda=\operatorname{diag}(s_1^2,\ldots,s_d^2).
\]
Set
\[
    A_1=U^\top A_0,
    \qquad
    \phi^{(1)}=A_1\phi=U^\top\phi^{(0)},
    \qquad
    \psi^{(1)}=A_1^{-\top}\psi=U^\top\psi^{(0)}.
\]
Multiplication by $U^\top$ diagonalizes the second matrix without changing
the first, because $U$ is orthogonal:
\[
    \Sigma_{\phi^{(1)}}=U^\top I U=I,
    \qquad
    \Sigma_{\psi^{(1)}}=U^\top B U=\Lambda.
\]
The coordinates are therefore aligned with the spectral directions, but
the pair remains unbalanced: its second-moment matrices are $(I,\Lambda)$.

\paragraph{Step 3: balancing.}
To see how to balance the pair, first note that multiplying $\psi^{(1)}$ by
$\Lambda^{-1/2}$ normalizes its second-moment matrix:
\[
    \Sigma_{\Lambda^{-1/2}\psi^{(1)}}
    =\Lambda^{-1/2}
      \Sigma_{\psi^{(1)}}
      \Lambda^{-1/2}
    =\Lambda^{-1/2}\Lambda\Lambda^{-1/2}
    =I.
\]
Rather than placing this entire normalization in the $\psi$ factor, we split
it equally between the two factors while preserving their product:
\[
    \tilde\phi=\Lambda^{1/4}\phi^{(1)},
    \qquad
    \tilde\psi=\Lambda^{-1/4}\psi^{(1)}.
\]
Their product is still unchanged, while their second-moment matrices become
\[
    \Sigma_{\tilde\phi}
    =\Lambda^{1/4}I\Lambda^{1/4}
    =\Lambda^{1/2},
    \quad
    \Sigma_{\tilde\psi}
    =\Lambda^{-1/4}\Lambda\Lambda^{-1/4}
    =\Lambda^{1/2}.
\]
Combining the normalization, diagonalization, and balancing transformations
gives
\[
    \tilde\phi=A\phi,
    \qquad
    A=\Lambda^{1/4}A_1
     =\Lambda^{1/4}U^\top\Sigma_\phi^{-1/2}.
\]
The corresponding transformation of the other factor is
$\tilde\psi=A^{-\top}\psi$, as claimed. Since the pointwise product was
preserved at every step,
{ $\kp_{\tilde\phi,\tilde\psi}=\kp_{\phi,\psi}$} and therefore
$L(\tilde\phi,\tilde\psi)=L(\phi,\psi)$.

\paragraph{Orthonormalization and singular values.}
Let
$
    D=\Lambda^{1/2}=\operatorname{diag}(s_1,\ldots,s_d)
$
and define
$\phi^\perp=D^{-1/2}\tilde\phi$ and
$\psi^\perp=D^{-1/2}\tilde\psi$. Then
\[
    \Ex{P_X}{\phi^\perp(X)\phi^\perp(X)^\top}
    =
    \Ex{P_Y}{\psi^\perp(Y)\psi^\perp(Y)^\top}
    =I,
\]
so the coordinates of $\phi^\perp$ and $\psi^\perp$ form orthonormal systems.
Moreover,
\[
    \tilde\phi(x)^\top\tilde\psi(y)
    =\phi^\perp(x)^\top D\psi^\perp(y)
    =\sum_{i=1}^d
      s_i\phi_i^\perp(x)\psi_i^\perp(y).
\]
This is a singular-value decomposition of
{ $\kp_{\phi,\psi}$}. Finally,
$\Sigma_\phi\Sigma_\psi$ is similar to
$B=\Sigma_\phi^{1/2}\Sigma_\psi\Sigma_\phi^{1/2}$, whose eigenvalues are
$s_i^2$. Hence
$s_i=\sqrt{\lambda_i(\Sigma_\phi\Sigma_\psi)}$, completing the proof.
\end{proof}

\subsection{Proofs for Section~\ref{sec:theoretical_guarantees}}
\label{app:proofs_theoretical_guarantees}

\begin{proof}[Proof of Lemma~\ref{lem:gap_characterization}]
Fix $\phi=\phi_t$. By the definitions of $A$ and $C$ in
Proposition~\ref{prop:population_loss},
\begin{align*}
    C(\phi_t,\psi)
    &=
    \Ex{P_X\otimes P_Y}{
      \psi(Y)^\top\phi_t(X)\phi_t(X)^\top\psi(Y)
    }=
    \Ex{P_Y}{
      \psi(Y)^\top\Sigma_{\phi,t}\psi(Y)
    },\\
    A(\phi_t,\psi)
    &=
    \Ex{P_Y}{
      \psi(Y)^\top\E{\phi_t(X)\mid Y}
    }.
\end{align*}
The last identity follows from the tower property, conditioning the
joint expectation in $A$ on $Y$. Since
$m_{\phi,t}(Y)=\E{\phi_t(X)\mid Y}$, the block objective is therefore
\[
    L(\phi_t,\psi)
    =
    \Ex{P_Y}{
      \psi(Y)^\top\Sigma_{\phi,t}\psi(Y)
      -2\psi(Y)^\top m_{\phi,t}(Y)
    }.
\]
The exact best response
$\psi_{t+1}^*=\Sigma_{\phi,t}^{-1}m_{\phi,t}$ satisfies
$m_{\phi,t}=\Sigma_{\phi,t}\psi_{t+1}^*$. Therefore,
\begin{align*}
    &\psi^\top\Sigma_{\phi,t}\psi
      -2\psi^\top m_{\phi,t}=
      (\psi-\psi_{t+1}^*)^\top
      \Sigma_{\phi,t}
      (\psi-\psi_{t+1}^*)
      -(\psi_{t+1}^*)^\top
       \Sigma_{\phi,t}\psi_{t+1}^*,
\end{align*}
and the final term is independent of $\psi$. Substituting
$\psi=\psi_{t+1}^*$ shows that
$L(\phi_t,\psi_{t+1}^*)=-\Ex{P_Y}{(\psi_{t+1}^*(Y))^\top\Sigma_{\phi,t}\psi_{t+1}^*(Y)}$,
so, for every $\psi\in L_0^2(P_Y)^d$,
\[
    L(\phi_t,\psi)-L(\phi_t,\psi_{t+1}^*)
    =
    \|\psi-\psi_{t+1}^*\|_{\Sigma_{\phi,t}}^2.
\]
Taking $\psi=\psi_t$ gives $G_t^\psi=\|r_t^\psi\|_{\Sigma_{\phi,t}}^2$. The
identical calculation, fixing $\psi=\psi_{t+1}$ and using
$\phi_{t+1}^*=\Sigma_{\psi,t+1}^{-1}m_{\psi,t+1}$, gives, for every
$\phi\in L_0^2(P_X)^d$,
\[
    L(\phi,\psi_{t+1})-L(\phi_{t+1}^*,\psi_{t+1})
    =
    \|\phi-\phi_{t+1}^*\|_{\Sigma_{\psi,t+1}}^2,
\]
and, at $\phi=\phi_t$, $G_t^\phi=\|r_t^\phi\|_{\Sigma_{\psi,t+1}}^2$. This
proves the quadratic representation and its specialization.

For the sandwich bound, write
$G_t^\psi=\|\Sigma_{\phi,t}^{1/2}(\psi_t-\psi_{t+1}^*)\|_{L^2(P_Y)}^2$.
Assumption~\ref{ass:nondegenerate_factors} gives
$\lambda I\preceq\Sigma_{\phi,t}\preceq\Lambda I$, hence
\[
    \lambda\|\psi_t-\psi_{t+1}^*\|_{L^2(P_Y)}^2
    \le G_t^\psi\le
    \Lambda\|\psi_t-\psi_{t+1}^*\|_{L^2(P_Y)}^2,
\]
and analogously for $G_t^\phi$.

For the gradient identity, $m_{\phi,t}=\Sigma_{\phi,t}\psi_{t+1}^*$ and
Proposition~\ref{prop:derivatives} give
\[
    \nabla_\psi L(\phi_t,\psi_t)
    =2\Sigma_{\phi,t}\psi_t-2m_{\phi,t}
    =2\Sigma_{\phi,t}(\psi_t-\psi_{t+1}^*),
\]
and analogously
$\nabla_\phi L(\phi_t,\psi_{t+1})=2\Sigma_{\psi,t+1}(\phi_t-\phi_{t+1}^*)$.
For any positive-definite $\Sigma\preceq\Lambda I$, diagonalizing
$\Sigma=U\operatorname{diag}(\lambda_1,\ldots,\lambda_d)U^\top$ gives
$\lambda_i^2\le\Lambda\lambda_i$ for every $i$
(Assumption~\ref{ass:nondegenerate_factors}), so $\Sigma^2\preceq\Lambda\Sigma$
and $z^\top\Sigma^2z\le\Lambda z^\top\Sigma z$ for every vector $z$. Hence
\[
    \|\nabla_\psi L(\phi_t,\psi_t)\|_{L^2(P_Y)}^2
    =4\bigl\|\Sigma_{\phi,t}(\psi_t-\psi_{t+1}^*)\bigr\|_{L^2(P_Y)}^2
    \le4\Lambda G_t^\psi,
\]
and analogously
$\|\nabla_\phi L(\phi_t,\psi_{t+1})\|_{L^2(P_X)}^2\le4\Lambda G_t^\phi$.
\end{proof}

\begin{proof}[Proof of Theorem~\ref{thm:weak_learner_convergence}]
Assumption~\ref{ass:nondegenerate_factors} gives
$
    \lambda I\preceq\Sigma_{\phi,t}\preceq\Lambda I,
    \,
    \lambda I\preceq\Sigma_{\psi,t+1}\preceq\Lambda I,
$
whereas Assumption~\ref{ass:relative_tree_accuracy} gives
\begin{align*}
    \|r_t^\psi-h_t^\psi\|_{\Sigma_{\phi,t}}^2
    \le(1-\gamma_\psi^2)
      \|r_t^\psi\|_{\Sigma_{\phi,t}}^2,\qquad \|r_t^\phi-h_t^\phi\|_{\Sigma_{\psi,t+1}}^2
    \le(1-\gamma_\phi^2)
      \|r_t^\phi\|_{\Sigma_{\psi,t+1}}^2.
\end{align*}

By Lemma~\ref{lem:gap_characterization}, for every
$\psi\in L_0^2(P_Y)^d$,
\[
    L(\phi_t,\psi)-L(\phi_t,\psi_{t+1}^*)
    =
    \|\psi-\psi_{t+1}^*\|_{\Sigma_{\phi,t}}^2,
\]
and, in particular, $G_t^\psi=\|r_t^\psi\|_{\Sigma_{\phi,t}}^2$.

We next compare the approximate update with the exact block minimizer. The
definitions of $r_t^\psi$ and $h_t^\psi$ give
\begin{align*}
    \psi_{t+1}-\psi_{t+1}^*
    &=\psi_t+\eta_\psi h_t^\psi-\psi_{t+1}^*=-r_t^\psi+\eta_\psi h_t^\psi=-(1-\eta_\psi)r_t^\psi
      +\eta_\psi(h_t^\psi-r_t^\psi).
\end{align*}
Since $\eta_\psi\in(0,1]$, the triangle inequality and
    Assumption~\ref{ass:relative_tree_accuracy} give
\begin{align*}
    \|\psi_{t+1}-\psi_{t+1}^*\|_{\Sigma_{\phi,t}}
    &\le
    (1-\eta_\psi)\|r_t^\psi\|_{\Sigma_{\phi,t}}
    +\eta_\psi
      \|h_t^\psi-r_t^\psi\|_{\Sigma_{\phi,t}}\\
    &\le
    \left(
      1-\eta_\psi
      +\eta_\psi\sqrt{1-\gamma_\psi^2}
    \right)
    \|r_t^\psi\|_{\Sigma_{\phi,t}}=q_\psi\|r_t^\psi\|_{\Sigma_{\phi,t}}.
\end{align*}
Moreover, $\gamma_\psi>0$ implies
$\sqrt{1-\gamma_\psi^2}<1$. Thus
$q_\psi=(1-\eta_\psi)+\eta_\psi\sqrt{1-\gamma_\psi^2}<1$.
Squaring the preceding inequality and using the quadratic representation
with $\psi=\psi_{t+1}$ yields
\[
    L(\phi_t,\psi_{t+1})-L(\phi_t,\psi_{t+1}^*)
    \le q_\psi^2G_t^\psi.
\]

Now fix $\psi=\psi_{t+1}$; Lemma~\ref{lem:gap_characterization}
gives, for every $\phi\in L_0^2(P_X)^d$,
\[
    L(\phi,\psi_{t+1})
      -L(\phi_{t+1}^*,\psi_{t+1})
    =
    \|\phi-\phi_{t+1}^*\|_{\Sigma_{\psi,t+1}}^2,
\]
and, in particular, $G_t^\phi=\|r_t^\phi\|_{\Sigma_{\psi,t+1}}^2$.
The approximate $\phi$-update satisfies
\begin{align*}
    \phi_{t+1}-\phi_{t+1}^*
    &=\phi_t+\eta_\phi h_t^\phi-\phi_{t+1}^*=-(1-\eta_\phi)r_t^\phi
      +\eta_\phi(h_t^\phi-r_t^\phi).
\end{align*}
Applying the triangle inequality and the $\phi$-accuracy condition exactly as
for the first block yields
\[
    \|\phi_{t+1}-\phi_{t+1}^*\|_{\Sigma_{\psi,t+1}}
    \le q_\phi
    \|r_t^\phi\|_{\Sigma_{\psi,t+1}}.
\]
After squaring and using the quadratic identity, we obtain
\[
    L(\phi_{t+1},\psi_{t+1})
      -L(\phi_{t+1}^*,\psi_{t+1})
    \le q_\phi^2G_t^\phi.
\]
Indeed, $\gamma_\phi>0$ gives
$\sqrt{1-\gamma_\phi^2}<1$, and hence
\[
    q_\phi
    =(1-\eta_\phi)+\eta_\phi\sqrt{1-\gamma_\phi^2}
    <(1-\eta_\phi)+\eta_\phi=1.
\]

We can now quantify the loss decrease. For the first block,
\begin{align*}
    &L(\phi_t,\psi_t)-L(\phi_t,\psi_{t+1})=
      \underbrace{
        L(\phi_t,\psi_t)-L(\phi_t,\psi_{t+1}^*)
      }_{G_t^\psi}
      -
      \underbrace{
        \left[
          L(\phi_t,\psi_{t+1})-L(\phi_t,\psi_{t+1}^*)
        \right]
      }_{\le q_\psi^2G_t^\psi}\\
    &\quad\ge (1-q_\psi^2)G_t^\psi.
\end{align*}
Analogously, the second block satisfies
\[
    L(\phi_t,\psi_{t+1})
      -L(\phi_{t+1},\psi_{t+1})
    \ge(1-q_\phi^2)G_t^\phi.
\]
Adding these two inequalities proves
\[
    L(\phi_t,\psi_t)-L(\phi_{t+1},\psi_{t+1})
    \ge
    (1-q_\psi^2)G_t^\psi+(1-q_\phi^2)G_t^\phi.
\]

Both coefficients are strictly positive, so the
loss is nonincreasing. It
is also bounded below: Proposition~\ref{prop:population_loss} gives
\[
    L(\phi,\psi)
    =
    \|\kp_0-\kp_{\phi,\psi}\|_{L^2(P_X\otimes P_Y)}^2
    -\|\kp_0\|_{L^2(P_X\otimes P_Y)}^2
    \ge -\|\kp_0\|_{L^2(P_X\otimes P_Y)}^2.
\]
It follows that $L(\phi_t,\psi_t)$ converges to a finite limit
$L_\infty$. Write $L_t\coloneqq L(\phi_t,\psi_t)$ and
$S_t\coloneqq G_t^\phi+G_t^\psi$. Since $a=\min\{1-q_\phi^2,1-q_\psi^2\}$,
the per-iteration decrease established above gives
\[
    aS_t
    \le
    (1-q_\psi^2)G_t^\psi+(1-q_\phi^2)G_t^\phi
    \le
    L_t-L_{t+1}.
\]

Summing this inequality from $t=0$ to $T-1$ telescopes to
$a\sum_{t<T}S_t\le L_0-L_T\le L_0-L_\infty$; since $S_t\ge0$,
$T\min_{t<T}S_t\le\sum_{t<T}S_t$, giving the first rate. For the gradient
rate, the gradient bound of Lemma~\ref{lem:gap_characterization}
gives, at $t_*\in\operatorname*{arg\,min}_{0\le t<T}S_t$,

\begin{align*}
    &\|\nabla_\phi L(\phi_{t_*},\psi_{t_*+1})\|_{L^2(P_X)}^2
    +\|\nabla_\psi L(\phi_{t_*},\psi_{t_*})\|_{L^2(P_Y)}^2\\
    &\qquad\le4\Lambda S_{t_*}
    =4\Lambda\min_{0\le t<T}S_t
    \le\frac{4\Lambda(L(\phi_0,\psi_0)-L_\infty)}{aT}.
\end{align*}

which proves the second rate, since the minimum of the gradient sum over
$0\le t<T$ cannot exceed its value at $t_*$.

Since $L_t\to L_\infty$, $L_t-L_{t+1}\to0$, so $0\le aS_t\le
L_t-L_{t+1}\to0$ forces $S_t\to0$ for the full sequence, i.e.
$G_t^\psi\to0$ and $G_t^\phi\to0$; by the gradient bound of
Lemma~\ref{lem:gap_characterization} again, the corresponding
gradients $\nabla_\psi L(\phi_t,\psi_t)$ and
$\nabla_\phi L(\phi_t,\psi_{t+1})$ vanish as well.

 Finally, suppose the iterates
$(\phi_t,\psi_t)\to(\phi_\infty,\psi_\infty)$ in product $L^2$. Then
$
    \phi_t\to\phi_\infty
$in $L^2(P_X)^d,$
    \,
    $\psi_t\to\psi_\infty$
    in $L^2(P_Y)^d.$
{ The shifted sequence also satisfies
$\psi_{t+1}\to\psi_\infty$ in $L^2(P_Y)^d$.}
We verify separately that every quantity appearing in the two gradient
formulas converges.

{ First consider the conditional mean terms. Set
$\delta_{\phi,t}\coloneqq\phi_t-\phi_\infty$. From the definition of
$m_\phi$,
\[
    m_{\phi,t}(Y)-m_{\phi,\infty}(Y)
    =
    \E{
      \delta_{\phi,t}(X)
      \,\middle|\,Y
    }.
\]
Conditional expectation is an orthogonal projection in $L^2$ and therefore
a contraction, so
\begin{align*}
    \|m_{\phi,t}-m_{\phi,\infty}\|_{L^2(P_Y)}
    &\le
    \|\delta_{\phi,t}\|_{L^2(P_X)}
    \longrightarrow0.
\end{align*}
{ Applying the same argument to
$\delta_{\psi,t+1}\coloneqq\psi_{t+1}-\psi_\infty$ gives}
\[
    \|m_{\psi,t+1}-m_{\psi,\infty}\|_{L^2(P_X)}
    \le
    { \|\psi_{t+1}-\psi_\infty\|_{L^2(P_Y)}}
    \longrightarrow0.
\]}

Next consider the second-moment matrices. For every pair of coordinates
$i,j$, adding and subtracting
$\Ex{P_X}{\phi_{\infty,i}(X)\phi_{t,j}(X)}$ and applying
Cauchy--Schwarz gives
\begin{align*}
    &\left|
      \Ex{P_X}{\phi_{t,i}(X)\phi_{t,j}(X)}
      -
      \Ex{P_X}{\phi_{\infty,i}(X)\phi_{\infty,j}(X)}
    \right|\\
    &\quad\le
      \|\phi_{t,i}-\phi_{\infty,i}\|_{L^2(P_X)}
      \|\phi_{t,j}\|_{L^2(P_X)}
      +
      \|\phi_{\infty,i}\|_{L^2(P_X)}
      \|\phi_{t,j}-\phi_{\infty,j}\|_{L^2(P_X)},
\end{align*}
and both terms tend to zero: the differences converge to zero in $L^2$,
whereas the convergent sequence $(\phi_{t,j})_t$ is bounded in $L^2$.
Thus every entry of $\Sigma_{\phi,t}$ converges to the corresponding entry
of $\Sigma_{\phi,\infty}$. Since these are finite $d\times d$ matrices,
entrywise convergence implies convergence in operator norm:
\[
    \|\Sigma_{\phi,t}-\Sigma_{\phi,\infty}\|_{\mathrm{op}}
    \longrightarrow0, \qquad \|\Sigma_{\psi,t+1}-\Sigma_{\psi,\infty}\|_{\mathrm{op}}
    \longrightarrow0
\]
In particular, both sequences of matrix operator norms are bounded.

We can now pass to the limit in each gradient formula. For the
$\psi$-gradient,
\begin{align*}
    &\|
      \Sigma_{\phi,t}\psi_t
      -\Sigma_{\phi,\infty}\psi_\infty
    \|_{L^2(P_Y)}\\
    &\quad\le
      \|\Sigma_{\phi,t}\|_{\mathrm{op}}
      \|\psi_t-\psi_\infty\|_{L^2(P_Y)}
      +
      \|\Sigma_{\phi,t}-\Sigma_{\phi,\infty}\|_{\mathrm{op}}
      \|\psi_\infty\|_{L^2(P_Y)}
      \longrightarrow0.
\end{align*}
Combining this with
$m_{\phi,t}\to m_{\phi,\infty}$ and using
Proposition~\ref{prop:derivatives}, we obtain
\[
    \nabla_\psi L(\phi_t,\psi_t)
    \longrightarrow
    \nabla_\psi L(\phi_\infty,\psi_\infty)
    \quad\text{in }L^2(P_Y)^d.
\]
Similarly,
\[
\|
      \Sigma_{\psi,t+1}\phi_t
      -\Sigma_{\psi,\infty}\phi_\infty
    \|_{L^2(P_X)}
      \longrightarrow0,
\]
and $m_{\psi,t+1}\to m_{\psi,\infty}$. Therefore,
\[
    { \nabla_\phi L(\phi_t,\psi_{t+1})}
    \longrightarrow
    \nabla_\phi L(\phi_\infty,\psi_\infty)
    \quad\text{in }L^2(P_X)^d.
\]
We already proved that the norms of the two gradient sequences on the
left converge to zero. Their $L^2$ limits must consequently be the zero
functions. Hence
\[
    \nabla_\psi L(\phi_\infty,\psi_\infty)=0,
    \qquad
    \nabla_\phi L(\phi_\infty,\psi_\infty)=0,
\]
so the limit is stationary.

For exact regressions, $\gamma_\phi=\gamma_\psi=1$. Therefore
$q_\phi=1-\eta_\phi$ and $q_\psi=1-\eta_\psi$, and
\[
    1-q_\phi^2=\eta_\phi(2-\eta_\phi),
    \qquad
    1-q_\psi^2=\eta_\psi(2-\eta_\psi).
\]
Taking the minimum of these two quantities gives the final expression for
$a$.
\end{proof}

\begin{proof}[Proof of Theorem~\ref{thm:population_landscape}]
{ Since the factors belong to the centered spaces, their conditional
means are the actions of the centered operator and its adjoint:}
\[
    m_\psi(x)=\E{\psi(Y)\mid X=x}=(\CEO_0\psi)(x),
\]
and, by the definition of the adjoint conditional expectation operator,
\[
    m_\phi(y)=\E{\phi(X)\mid Y=y}=(\CEO_0^*\phi)(y).
\]
At a stationary point, both gradients in
Proposition~\ref{prop:derivatives} vanish. The first gradient equation
therefore gives
\[
    0=2\Sigma_\psi\phi-2m_\psi
    \quad\Longrightarrow\quad
    \CEO_0\psi=\Sigma_\psi\phi,
\]
while the second gives
\[
    0=2\Sigma_\phi\psi-2m_\phi
    \quad\Longrightarrow\quad
    \CEO_0^*\phi=\Sigma_\phi\psi.
\]
Since $\Sigma_\phi=\Sigma_\psi=D$, these equations become
\[
    \CEO_0\psi=D\phi,
    \qquad
    \CEO_0^*\phi=D\psi.
\]
Because $D\succ0$, the normalized factors are well defined and
$\phi=D^{1/2}\phi^\perp$,
$\psi=D^{1/2}\psi^\perp$. The operators $\CEO_0$ and $\CEO_0^*$ act on
each coordinate and the constant matrix $D^{1/2}$ can be taken through
them. Thus
\[
    D^{1/2}\CEO_0\psi^\perp
    =D^{3/2}\phi^\perp,
    \quad
    D^{1/2}\CEO_0^*\phi^\perp
    =D^{3/2}\psi^\perp.
\]
Multiplication by $D^{-1/2}$ gives
\[
    \CEO_0\psi^\perp=D\phi^\perp,
    \qquad
    \CEO_0^*\phi^\perp=D\psi^\perp.
\]
In coordinates, these identities read
\[
    \CEO_0\psi_i^\perp=s_i\phi_i^\perp,
    \qquad
    \CEO_0^*\phi_i^\perp=s_i\psi_i^\perp,
    \qquad i=1,\ldots,d.
\]
Moreover, balancing gives
\[
    \Sigma_{\phi^\perp}
    =D^{-1/2}\Sigma_\phi D^{-1/2}=I,
    \quad
    \Sigma_{\psi^\perp}
    =D^{-1/2}\Sigma_\psi D^{-1/2}=I.
\]
Hence the coordinates of $\phi^\perp$ and $\psi^\perp$ are orthonormal
in $L^2(P_X)$ and $L^2(P_Y)$, respectively. The preceding coordinate
equations therefore show that
$(\phi_i^\perp,\psi_i^\perp,s_i)$ is a singular triplet of $\CEO_0$ for
every $i$.

It remains to determine which of these stationary points can be local
minima. The preceding argument shows that every balanced, nondegenerate
stationary point selects $d$ singular triplets of $\CEO_0$, but
stationarity alone does not guarantee that they are the leading ones.

The idea is clearest when $d=1$. In this case,
\[
    L(\phi,\psi)
    =\norm{\phi}_{L^2(P_X)}^2\norm{\psi}_{L^2(P_Y)}^2
      -2\langle\phi,\CEO_0\psi\rangle_{L^2(P_X)}.
\]
Suppose that the stationary point corresponds to a singular triplet
$(u,v,s)$, so that $\phi=\sqrt{s}u$ and $\psi=\sqrt{s}v$. If this triplet
is not leading, there is an orthogonal singular triplet
$(u^\star,v^\star,\sigma)$ with $\sigma>s$. For $\varepsilon\in\R$, define
\[
    \phi_\varepsilon=\sqrt{s}u+\varepsilon u^\star,
    \qquad
    \psi_\varepsilon=\sqrt{s}v+\varepsilon v^\star.
\]
These factors are centered and converge to $(\phi,\psi)$ in product $L^2$
as $\varepsilon\to0$. Orthogonality and the singular equations give
\[
    \norm{\phi_\varepsilon}_{L^2(P_X)}^2
    =\norm{\psi_\varepsilon}_{L^2(P_Y)}^2=s+\varepsilon^2,
    \qquad
    \langle\phi_\varepsilon,\CEO_0\psi_\varepsilon
    \rangle_{L^2(P_X)}=s^2+\sigma\varepsilon^2.
\]
Consequently,
\begin{align*}
    L(\phi_\varepsilon,\psi_\varepsilon)-L(\phi,\psi)
    &=(s+\varepsilon^2)^2
      -2(s^2+\sigma\varepsilon^2)+s^2\\
    &=2(s-\sigma)\varepsilon^2+\varepsilon^4<0
\end{align*}
whenever $0<\varepsilon^2<2(\sigma-s)$. Hence a nonleading singular
triplet cannot be a local minimum.

The same idea extends directly to $d>1$.

\end{proof}

\subsection{Finite-sample proof for
Section~\ref{sec:finite_sample_guarantee}}
\label{app:finite_sample_proofs}

\begin{proof}[Proof of Theorem~\ref{thm:finite_sample_generalization}]
First, we justify the convex-hull containment used in
Section~\ref{sec:finite_sample_guarantee}. For any scalar coordinate, write
the damped update as
\[
    f_{t+1}=(1-\eta_t)f_t+\eta_t h_t,
    \qquad \eta_t\in[0,1],
\]
where $f_0,h_0,\ldots,h_{T-1}\in\mathcal H_Z$. Expanding the recursion gives
\[
    f_T=
    \left\{\prod_{s=0}^{T-1}(1-\eta_s)\right\}f_0
    +\sum_{t=0}^{T-1}
      \left\{\eta_t\prod_{s=t+1}^{T-1}(1-\eta_s)\right\}h_t.
\]
All coefficients are nonnegative and sum to one, so
$f_T\in\operatorname{conv}(\mathcal H_Z)=\mathcal F_Z$.

Rademacher complexity is unchanged by taking a convex hull \citep{bartlett2002rademacher}, so
$
    \mathfrak R_n(\mathcal F_X)=\mathfrak R_n(\mathcal H_X),
    \,
    \mathfrak R_n(\mathcal F_Y)=\mathfrak R_n(\mathcal H_Y).
$
Moreover, every function in
$\mathcal F_X\cup\mathcal F_Y$ is bounded by $B$. To make the product-class
step explicit, define
\[
    \mathcal G_Z^{(2)}
    \coloneqq
    \{z\mapsto f(z)\widetilde f(z):
      f,\widetilde f\in\mathcal F_Z\},
    \qquad Z\in\{X,Y\},
\]
and
\[
    \mathcal G_{XY}^{(\times)}
    \coloneqq
    \{(x,y)\mapsto f(x)g(y):
      f\in\mathcal F_X,\ g\in\mathcal F_Y\}.
\]
Using
$uv=\{(u+v)^2-(u-v)^2\}/4$, each product class is a difference of two
squared sum classes. The scalar square map is $4B$-Lipschitz on
$[-2B,2B]$, so the scalar contraction inequality, together with
$\mathfrak R_n(\mathcal F+\widetilde{\mathcal F})
\le\mathfrak R_n(\mathcal F)+\mathfrak R_n(\widetilde{\mathcal F})$,
gives, for a universal constant $c_0$,
\begin{align*}
    \mathfrak R_n(\mathcal G_Z^{(2)})
    &\le c_0B\mathfrak R_n(\mathcal H_Z),\\
    \mathfrak R_n(\mathcal G_{XY}^{(\times)})
    &\le c_0B\bigl\{
      \mathfrak R_n(\mathcal H_X)+
      \mathfrak R_n(\mathcal H_Y)\bigr\}.
\end{align*}
Here we used the preceding convex-hull identity in both bounds. Applying the
standard Rademacher uniform law to the two first-moment classes and the three
product classes, followed by a union bound, yields a quantity
\[
    \varepsilon_n
    \le C_B\Delta_n(\delta)
\]
such that, with probability at least $1-\delta$, simultaneously for
$Z\in\{X,Y\}$,
\begin{align*}
    \sup_{f\in\mathcal F_Z}
    |(P_{n,Z}-P_Z)f|
    &\le\varepsilon_n,\\
    \sup_{f,\widetilde f\in\mathcal F_Z}
    |(P_{n,Z}-P_Z)(f\widetilde f)|
    &\le\varepsilon_n,\\
    \sup_{\substack{f\in\mathcal F_X\\g\in\mathcal F_Y}}
    |(P_n-P)f(X)g(Y)|
    &\le\varepsilon_n.
\end{align*}
The bounded-difference term in this uniform law is of order
$B^2\sqrt{\log(8d^2/\delta)/n}$ for the product classes and of order
$B\sqrt{\log(8d^2/\delta)/n}$ for the first-moment classes; both are absorbed
into $C_B\Delta_n(\delta)$. See, for example,
\citet[Chapter~4]{wainwright2019high}. Here $P_{n,Z}$ and $P_Z$ denote the
empirical and population marginal measures, while $P_n$ and $P$ denote the
empirical and population joint measures.

We next propagate these scalar bounds through the objective. For any
$\phi\in\mathcal F_X^d$, every entry of
$\widehat\Sigma_\phi-\Sigma_\phi$ has absolute value at most
$\varepsilon_n$, and the same holds for $\psi$. All empirical and population
second-moment entries have absolute value at most $B^2$. Therefore,
writing out the trace entrywise,
\begin{align*}
    \left|
      \tr{\widehat\Sigma_\phi\widehat\Sigma_\psi}
      -\tr{\Sigma_\phi\Sigma_\psi}
    \right|
    &\le
    \sum_{j,k=1}^d
      |\widehat\Sigma_{\phi,jk}-\Sigma_{\phi,jk}|
        |\widehat\Sigma_{\psi,kj}|
    \\
    &\quad+
    \sum_{j,k=1}^d
      |\Sigma_{\phi,jk}|
      |\widehat\Sigma_{\psi,kj}-\Sigma_{\psi,kj}|
    \le2d^2B^2\varepsilon_n.
\end{align*}
Since the classes are centered, $P_Xf=P_Yg=0$ for every
$f\in\mathcal F_X$ and $g\in\mathcal F_Y$. The first-moment bound thus gives
$|\overline\phi_{n,j}|,|\overline\psi_{n,j}|\le\varepsilon_n$; also, each
empirical mean has absolute value at most $B$. Hence
\[
    2|\overline\phi_n^\top\overline\psi_n|
    \le2dB\varepsilon_n.
\]
Finally, applying the joint-product bound coordinatewise yields
\[
    \left|
      \frac1n\sum_{i=1}^n\phi(X_i)^\top\psi(Y_i)
      -\E{\phi(X)^\top\psi(Y)}
    \right|
    \le d\varepsilon_n.
\]
Combining the last three displays, including the coefficient $2$ on the
joint term in the loss, gives uniformly over the vector-valued classes
\[
    |\widehat L_n(\phi,\psi)-L(\phi,\psi)|
    \le
    \bigl(2d^2B^2+2dB+2d\bigr)\varepsilon_n
    \le C_{d,B}\Delta_n(\delta),
\]
which proves the uniform bound. 

\end{proof}

\section{Experimental details}
\label{app:experiments}

Full implementations, exact hyperparameter grids, random seeds, and data
splits for every experiment are provided as section-specific notebooks in the
repository at \url{https://github.com/thiagorr162/fsnm}. Each notebook writes
its figures only to the code repository; figures included in the paper are
copied to the TeX directory after the reported outputs have been checked.
The repository README lists the complete configurations and baseline grids.
Unless stated otherwise, FSNM uses coordinatewise regression trees, ridge
$10^{-8}$ in the empirical Newton matrices, and empirical balancing after
every iteration.

\paragraph{Synthetic kernel recovery.}
The close-spectrum experiment uses $10{,}000/4{,}000$
training/validation pairs, rank $3$, step size $0.1$, depth $3$, and minimum
leaf size $300$; validation selects $11$ of $40$ candidate iterations, and
errors use a $160\times160$ grid. The regional and twenty-dimensional
experiments use $8{,}000/3{,}000$ pairs, rank $3$, step size $0.15$, depth
$3$, and at most $40$ iterations. Their minimum leaf sizes are $120$ and
$250$, and validation selects iterations $15$ and $32$, respectively. The
regional evaluation uses a $240\times240$ grid; the tabular evaluation uses
$10{,}000$ fresh product-marginal pairs.

\paragraph{Baselines and repeated samples.}
Tables~\ref{tab:rank3_baselines}--\ref{tab:tabular_baselines} use ten
independent draws. ACE uses the same tree budgets as FSNM and selects its
alternating iterations by validation correlation. uLSIF selects Gaussian
bandwidth and ridge by held-out squared loss. Kernel CCA uses $800$ random
landmarks and selects bandwidth and regularization by validation canonical
correlation. Reported intervals are
$\overline m\pm t_{9,0.975}s/\sqrt{10}$; exact grids and replicate seeds are
listed in the README.

\paragraph{Synthetic conditional CDF.}
For Section~\ref{sec:conditional_cdf_uncertainty}, every bootstrap replicate
resamples the $10{,}000$ paired training observations, refits the rank-three
model with the selected configuration, and evaluates
$n^{-1}\sum_i\widehat\kp(x,Y_i)\Ind{Y_i\leq t}$ on a grid of $301$
thresholds. The complete signed curve is then projected by anchored isotonic
regression onto the set of monotone $[0,1]$-valued curves. Pointwise bands are
the $2.5\%$ and $97.5\%$ quantiles across $100$ such projected replicates.

\paragraph{Dependence detection.}
Both settings use $10{,}000/4{,}000/4{,}000$
training/validation/test observations, rank $3$, step size $0.1$, depth $3$,
and minimum leaf size $300$. Validation selects one iteration under
independence and $13$ under nonlinear dependence. The held-out statistic is
minus the empirical FSNM loss; its null distribution uses $999$ permutations
of the fitted $Y$-factors and the plus-one correction for the $p$-value.

\paragraph{Restricted glass family.}
For Section~\ref{sec:glass_experiment}, we retain a record only when RI lies
in $[1,4.5]$, at least one of the ten selected oxide fractions is positive,
and the sum of the absolute fractions of every unselected oxide is below
$10^{-10}$. The selected fractions are renormalized to sum to one. A fixed
permutation with seed $2026$ gives $2{,}109/452/452$
training/validation/test observations. The validation grid contains ranks
$\{5,10,20\}$, tree depths $\{3,5\}$, and minimum leaf sizes
$\{25,50,100\}$; all fits use step size $0.1$, at most $80$ iterations, and
patience $10$. Both validation loss and CRPS select rank $20$, depth $5$,
leaf size $25$, and iteration $39$. Downstream expectations use the training
responses as the empirical marginal and apply the fitted kernel directly to
each query. Complete CDF curves are projected by isotonic regression, after
which quantile levels are recalibrated from validation PIT values; the test
partition is not used in either operation.

\end{document}